\pdfoutput=1
\documentclass[10pt]{article}
\usepackage{pdfpages,fullpage}
\usepackage[linesnumbered, ruled]{algorithm2e}
\usepackage{algorithmic}
\usepackage{multirow}
\usepackage{enumitem}
\usepackage{subcaption}
\usepackage{amsthm}

\usepackage{pifont}
\newcommand{\cmark}{\ding{51}}%
\newcommand{\xmark}{\ding{55}}%

\usepackage{Definitions}

\usepackage{color}

\newcommand{\bruce}[1]{}
\newcommand{\bo}[1]{}

\newtheorem*{rep@theorem}{\rep@title}
\newenvironment{oneshot}[1]{\def\rep@title{#1} \begin{rep@theorem}}{\end{rep@theorem}}
\newcommand{\what}{\widehat{w}}

\title{Scale Up Nonlinear Component Analysis with \\Doubly Stochastic Gradients}

\author{
Bo Xie\thanks{College of Computing, Georgia Institute of Technology. Email:\texttt{bo.xie@gatech.edu}}, ~
Yingyu Liang\thanks{Department of Computer Science, Princeton University. Email: \texttt{yingyul@cs.princeton.edu}},~
Le Song\thanks{College of Computing, Georgia Institute of Technology. Email:\texttt{lsong@cc.gatech.edu}}
}
\date{}

\begin{document}

\maketitle
\begin{abstract}
Nonlinear component analysis such as kernel Principle Component Analysis (KPCA) and kernel Canonical Correlation Analysis (KCCA) are widely used in machine learning, statistics and data analysis, but they can not scale up to big datasets. Recent attempts have employed random feature approximations to convert the problem to the primal form for linear computational complexity. However, to obtain high quality solutions, the number of random features should be the same order of magnitude as the number of data points, making such approach not directly applicable to the regime with millions of data points.

We propose a simple, computationally efficient, and memory friendly algorithm based on the ``doubly stochastic gradients'' to scale up a range of kernel nonlinear component analysis, such as kernel PCA, CCA and SVD. Despite the \emph{non-convex} nature of these problems, our method enjoys theoretical guarantees that it converges at the rate $\Otil(1/t)$ to the global optimum, even for the top $k$ eigen subspace. Unlike many alternatives, our algorithm does not require explicit orthogonalization, which is infeasible on big datasets. We demonstrate the effectiveness and scalability of our algorithm on large scale synthetic and real world datasets.
\end{abstract}

\section{Introduction}
\label{sec:intro}
Scaling up nonlinear component analysis has been challenging due to prohibitive computation and memory requirements. Recently, methods such as Randomized Component Analysis~\cite{LopSraSmoGhaSch14} are able to scale to larger datasets by leveraging random feature approximation. Such methods approximate the kernel function by using explicit random feature mappings, then perform subsequent steps in the primal form, resulting in linear computational complexity. Nonetheless, theoretical analysis~\cite{RahRec09,LopSraSmoGhaSch14} shows that in order to get high quality results, the number of random features should grow linearly with the number of data points. Experimentally, one often sees that the statistical performance of the algorithm improves as one increases the number of random features. 

Another approach to scale up the kernel component analysis is to use stochastic gradient descent and online updates~\cite{Oja82, Oja83}.  These stochastic methods have also been extended to the kernel case~\cite{KimFraSch05,ChiSut07,Honeine12}.  They require much less computation than their batch counterpart, converge in $O(1/t)$ rate, and are naturally applicable to streaming data setting.  Despite that, they share a severe drawback: all data points used in the updates need to be saved, rendering them impractical for large datasets.

In this paper, we propose to use the ``doubly stochastic gradients'' for nonlinear component analysis. This technique is a general framework for scaling up kernel methods~\cite{DaiXieHe14} for {\em convex problems} and has been successfully applied to many popular kernel machines such as kernel SVM, kernel ridge regressions, and Gaussian process. It uses two types of stochastic approximation simultaneously: random data points instead of the whole dataset (as in stochastic update rules), and random features instead of the true kernel functions (as in randomized component analysis).
These two approximations lead to the following benefits: 
\begin{itemize}
\item {\bf Computation efficiency} 
The key computation is the generation of a mini-batch of random features and the evaluation of them on a mini-batch of data points, which is very efficient.
\item {\bf Memory efficiency} 
Instead of storing training data points, we just keep a small program for regenerating the random features, and sample previously used random features according to pre-specified random seeds. This leads to huge savings: the memory requirement up to step $t$ is $O(t)$, independent of the dimension of the data. 
\item {\bf Adaptibility} 
Unlike other approaches that can only work with a fixed number of random features beforehand, doubly stochastic approach is able to increase the model complexity by using more features when new data points arrive, and thus enjoys the advantage of nonparametric methods.
\end{itemize}

Although on first look our method appears similar to the approach in~\cite{DaiXieHe14}, the two methods are fundamentally different. In~\cite{DaiXieHe14}, they address \emph{convex problems}, whereas our problem is highly \emph{non-convex}. The convergence result in~\cite{DaiXieHe14} crucially relies on the properties of convex functions, which do not translate to our problem. Instead, our analysis centers around the stochastic update of power iterations, which uses a different set of proof techniques.


In this paper, we make the following contributions.
\begin{itemize}
\item {\bf General framework}  
We show that the general framework of doubly stochastic updates can be applied in various kernel component analysis tasks, including KPCA, KSVD, KCCA, \etc. 
\item {\bf Strong theoretical guarantee}
We prove that the finite time convergence rate of doubly stochastic approach is $\Otil(1/t)$. This is a significant result since 1) the \emph{global} convergence result is w.r.t. a \emph{non-convex problem}; 2) the guarantee is for update rules without explicit orthogonalization. Previous works require explicit orthogonalization, which is impractical for kernel methods on large datasets.
\item {\bf Strong empirical performance} Our algorithm can scale to datasets with millions of data points. Moreover, the algorithm can often find much better solutions thanks to the ability to use many more random features.
We demonstrate such benefits on both synthetic and real world datasets.
\end{itemize}

Since kernel PCA is a typical task, we focus on it in the paper and provide a description of other tasks in Section~\ref{sec:extensions}.
Although we only state the guarantee for kernel PCA, the analysis naturally carries over to the other tasks.

\section{Related work} 
\label{sec:rel_work}

Many efforts have been devoted to scale up kernel methods. The random feature approach~\cite{RahRec08,RahRec09} approximates the kernel function with explicit random feature mappings and solves the problem in primal form, thus circumventing the quadratic computational complexity. It has been applied to various kernel methods~\cite{LeSarSmo13,DaiXieHe14,LopSraSmoGhaSch14}, among which most related to our work is Randomized Component Analysis~\cite{LopSraSmoGhaSch14}.
One drawback of Randomized Component Analysis is that their theoretical guarantees are only for kernel matrix approximation: it does not say anything about how close the solution obtained from randomized PCA is to the true solution. In contrast, we provide a finite time convergence rate of how our solution approaches the true solution.
In addition, even though a moderate size of random features can work well for tens of thousands of data points, datasets with tens of millions of data points require many more random features. Our online approach allows the number of random features, hence the flexibility of the function class, to grow with the number of data points. This makes our method suitable for data streaming setting, which is not possible for previous approaches.

Online algorithms for PCA have a long history.  Oja proposed two stochastic update rules for approximating the first eigenvector and provided convergence proof in~\cite{Oja82, Oja83}, respectively. These rules have been extended to the generalized Hebbian update rules~\cite{Sanger89b,SchGunVis07,BalDasFre13} that compute the top $k$ eigenvectors (the subspace case). Similar ones have also been derived from the perspective of optimization and stochastic gradient descent~\cite{SchGunVis07,AroCotSre13}.
They are further generalized to the kernel case~\cite{KimFraSch05,ChiSut07,Honeine12}. However, online kernel PCA needs to store all the training data, which is impractical for large datasets. Our doubly stochastic method avoids this problem by using random features and keeping only a small program for regenerating previously used random features according to pre-specified seeds. As a result, it can scale up to tens of millions of data points.

For finite time convergence rate, \cite{BalDasFre13} proved the $O(1/t)$ rate for the top eigenvector in linear PCA using Oja's rule. For the same task, \cite{Shamir14} proposed a noise reduced PCA with linear convergence rate, where the rate is in terms of epochs, \ie, number of passes over the whole dataset. The noisy power method presented in~\cite{HarPri14} provided linear convergence for a subspace, although it only converges linearly to a constant error level. In addition, the updates require explicit orthogonalization, which is impractical for kernel methods. In comparison, our method converges in $O(1/t)$ for a subspace, without the need for orthogonalization.

\section{Preliminaries}
\label{sec:prelim}
\subsection{Kernels and Covariance Operators}

A kernel $k(x, y):\Xcal \times \Xcal \mapsto \RR$ is a function that is positive-definite (PD), i.e.,  for all $n > 1$, $c_1, \dots, c_n \in \RR$, and $x_1, \dots, x_n \in \Xcal$, we have
\begin{align*}
\sum_{i,j=1}^n c_i c_j k(x_i, x_j) \ge 0 .
\end{align*}

A reproducing kernel Hilbert space (RKHS) $\Fcal$ on $\Xcal$ is a Hilbert space of functions from $\Xcal$ to $\RR$. $\Fcal$ is an RKHS if and only if there exists a $k(x,x'):\Xcal\times \Xcal \mapsto \RR$ such that
  $
    \forall x \in \Xcal, k(x,\cdot) \in \Fcal,~\text{and}~
    \forall f \in \Fcal, \inner{f(\cdot)}{k(x,\cdot)}_{\Fcal} = f(x).
  $
If such a $k(x,x')$ exist, it is unique and it is a PD kernel.
A function $f \in \Fcal$ if and only if $\nbr{f}_{\Fcal}^2 := \inner{f}{f}_{\Fcal} < \infty$.

Given a distribution $\PP(x)$, a kernel function $k(x,x')$ with RKHS $\Fcal$, the covariance operator $A:\Fcal\mapsto \Fcal$ is a linear self-adjoint operator defined as 
\begin{align}
  \label{eq:cov}
  A f(\cdot) := \EE_{x}[f(x)\, k(x,\cdot)],\quad \forall f \in \Fcal,  
\end{align}
and furthermore $\inner{g}{Af}_{\Fcal} = \EE_{x}[f(x)\, g(x)]$, $\forall g \in \Fcal$.

Let $F = \rbr{f_1(\cdot), f_2(\cdot), \dots, f_k(\cdot)}$ be a list of $k$ functions in the RKHS, and we define matrix-like notation
\begin{align}
  A F(\cdot) := \rbr{A f_1(\cdot), \dots, A f_k(\cdot)},   
\end{align}
and $F^\top A F$ is a $k \times k$ matrix, whose $(i, j)$-th element is $\inner{f_i}{A f_j}_{\Fcal}$.
The outer-product of a function $v\in \Fcal$ defines a linear operator $v v^\top:\Fcal\mapsto\Fcal$ such that
\begin{align}
  (v v^\top)f(\cdot)  := \inner{v}{f}_{\Fcal} v(\cdot),\quad \forall f \in \Fcal
\end{align}
Let $V =  \rbr{v_1(\cdot), \dots, v_k(\cdot)}$ be a list of $k$ functions, then 
the weighted sum of a set of linear operators, $\cbr{v_i v_i^\top}_{i=1}^k$, can be denoted using matrix-like notation as
\begin{align}
  V \Sigma_k V^\top := \sum_{i=1}^k \lambda_i v_i v_i^\top 
\end{align}
where $\Sigma_k$ is a diagonal matrix with $\lambda_i$ on the $i$-th entry of the diagonal.

\subsection{Kernel PCA}

Kernel PCA aims to identify the top $k$ eigenfunctions $V =  \rbr{v_1(\cdot), \dots, v_k(\cdot)}$ for the covariance operator $A$, where $V$ is also called the top $k$ subspace for $A$. 

A function $v$ is an eigenfunction of covariance operator $A$ with the corresponding eigenvalue $\lambda$ if 
\begin{align}
  Av(\cdot) = \lambda v(\cdot). 
\end{align}
Given a set of eigenfunctions $\cbr{v_i}$ and associated eigenvalues $\cbr{\lambda_i}$, where $\inner{v_i}{v_j}_{\Fcal} = \delta_{ij}$. We can denote the eigenvalue of $A$ as 
\begin{align}
  A =  V\Sigma_k V^\top + V_{\perp}\Sigma_{\perp} V_{\perp}^\top 
\end{align}
where $V =  \rbr{v_1(\cdot), \dots, v_k(\cdot)}$ is the top $k$ eigenfunctions of $A$, and $\Sigma_k$ is a diagonal matrix with the corresponding eigenvalues, $V_{\perp}$ is the collection of the rest of the eigenfunctions, and $\Sigma_{\perp}$ is a diagonal matrix with the rest of the eigenvalues. 

In the finite data case, the empirical covariance operator is $A = \frac{1}{n}\sum_i k(x_i, \cdot) k(x_i, \cdot)^\top$ or denoted as $\frac{1}{n}\sum_i k(x_i, \cdot) \otimes k(x_i, \cdot)$. According to the representer theorem, the solutions of the top $k$ eigenfunctions of $A$ can be expressed as linear combinations of the training points with the set of coefficients $\cbr{\alpha_i}_{i=1}^k \in \RR^n$,
\begin{align*}
v_i = \sum_{j=1}^n \alpha_i^j k(x_j, \cdot)
\end{align*}
Using $Av(\cdot) = \lambda v(\cdot)$ and the kernel trick, we have
\begin{align*}
K\alpha_i = \lambda_i \alpha_i,
\end{align*}
where $K$ is the $n \times n$ Gram matrix.

The infinite dimensional problem is thus reduced to a finite dimensional eigenvalue problem. However, this dual approach is clearly impractical on large scale datasets due quadratic memory and computational costs.

\begin{table*}[t!]
  \setlength{\tabcolsep}{2pt}
  \centering
  \caption{Example of kernels and their random feature representation}\label{table:explicit_features}
    {\small
    \begin{tabular}{ll|c|c|c}
      \hline
      \hline 
      &Kernel &$k(x, x')$ &$\phi_{\omega}(x)$ &$p(\omega)$\\
      \hline
      &Gaussian~\cite{RahRec08} &$\exp(-\frac{\|x - x'\|_2^2}{2})$ &$\exp(-i\omega^\top x)$ &${2\pi}^{-\frac{d}{2}}\exp(-\frac{\|\omega\|_2^2}{2})$  \\ 
      &Laplacian~\cite{RahRec08} &$\exp(-{\|x - x'\|_1})$ &$\exp(-i\omega^\top x)$ &$\prod_{i=1}^d \frac{1}{\pi(1+\omega_i^2)}$ \\ 
      &Cauchy~\cite{RahRec08}   &$\prod_{i=1}^d\frac{2}{1+(x_i -x'_i )^2}$ &$\exp(-i\omega^\top x)$ & $\exp(-{\|\omega\|_1})$ \\ 
      &Mat{\'e}rn~\cite{RasWil06} &$\frac{2^{1-\nu}}{\Gamma(\nu)}\rbr{\frac{\sqrt{2\nu}\|x - x'\|_2}{\ell}}^\nu K_\nu\rbr{\frac{\sqrt{2\nu}\|x - x'\|_2}{\ell}}$ &$\exp(-i\omega^\top x)$ &$\frac{2^d\pi^{d/2}\Gamma(\nu + d/2)(2\nu)^\nu}{\Gamma(\nu)\ell^{2\nu}}\rbr{\frac{2\nu}{\ell^2} + 4\pi^2\|\omega\|_2^2}^{\nu + d/2}$ \\
      &Dot Product~\cite{KarKar12} &$\sum_{n=0}^\infty a_n\langle x, x'\rangle^n\quad a_n\ge 0$ &$\sqrt{a_N p^{N+1}}\prod_{i=1}^N\omega_i^\top x$ &$\PP[N=n] = \frac{1}{p^{n+1}}$\\
      &Polynomial~\cite{PhaPag13} &$(\langle x, x'\rangle + c)^p$ &$\mathtt{FFT}^{-1}(\odot_{i=1}^p \mathtt{FFT}(C_i x))$ &$C_j=S_jD_j,\quad D_j\in \RR^{d\times d}\quad S_j\in \RR^{D\times d}$\\
      &\small{Exp-Semigroup}~\cite{YanSinFanAvretal14} &$\exp(-\beta \sum_{i=1}^d \sqrt{x_i + x'_j})$ &$\exp(-\omega^\top x)$ &$\prod_{i=1}^d\frac{\beta}{2\sqrt{\pi}}\omega_i^{-\frac{3}{2}}\exp(-\frac{\beta}{4\omega_i})$ \\
      &\small{Rec-Semigroup}~\cite{YanSinFanAvretal14} &$\prod_{i=1}^d\frac{\lambda}{x_i + x'_i + \lambda}$ &$\exp(-\omega^\top x)$ &$\prod_{i=1}^d\lambda\exp(-\lambda\omega_i)$ \\
      &Arc-Cosine~\cite{ChoSaul09} &$\frac{1}{\pi}\|x\|^n\|x'\|^n J_n(\theta) $ &$(\omega^\top x)^n\max(0, \omega^\top x)$ &${2\pi}^{-\frac{d}{2}}\exp(-\frac{\|\omega\|_2^2}{2})$\\
      \hline
      \hline 
    \end{tabular}
    \text{\scriptsize 
    $D_j$ is random $\{\pm 1\}$ diagonal matrix and the columns of $S_j$ are uniformly selected from $\{e_1,\ldots, e_D\}$. $\nu$ and $\ell$ are positive parameters.}\\
    \text{\scriptsize $K_\nu$ is a modified Bessel function. $\odot$ stands for element-wise product. $\theta = \cos^{-1}\frac{x^\top x'}{\|x\|\|x'\|}$, $J_n(\theta) = (-1)^n(\sin\theta)^{n+1}\rbr{\frac{\partial}{\partial \theta}}^n\rbr{\frac{\pi - \theta}{\sin\theta}}$}
    }
    \vspace{-3mm}
\end{table*}

\subsection{Random feature approximation}
The usage of random features to approximate a kernel function is motivated by the following theorem.
\begin{theorem}[Bochner]
  A continuous, real-valued, symmetric and shift-invariant function $k(x-x')$ on $\RR^d$ is a PD kernel if and only if there is a finite non-negative measure $\PP(\omega)$ on $\RR^d$, such that
  $
    k(x-x') = \int_{\RR^d} \, e^{i \omega^\top (x-x')}\, d\PP(\omega) = \int_{\RR^d \times [0,2\pi]} \phi_{\omega}(x) \phi_{\omega}(y)\, d \rbr{\PP(\omega) \times \PP(b)},
  $
  where $\PP(b)$ is a uniform distribution on $[0,2\pi]$, and $\phi_{\omega}(x) = \sqrt{2}\cos(\omega^\top x + b)$.
\end{theorem}

The theorem says that any shift-invariant kernel function $k(x,y) = k(x-y)$, \eg, Gaussian RBF kernel, can be considered as an expectation of two feature functions $\phi_{\omega}(x)$ and $\phi_{\omega}(y)$, where the expectation is taked over a distribution on the random frequency $\omega$ and phase $b$.

We can therefore approximate the kernel function as an empirical average of samples from the distribution. In other words, 
\begin{align*}
k(x, y) \approx \frac{1}{B} \sum_i \phi_{\omega_i}(x) \phi_{\omega_i}(y),
\end{align*}
where $\cbr{\rbr{\omega_i, b_i}}_i^B$ are i.i.d. samples drawn from from $\PP(\omega)$ and $\PP(b)$, respectively. 

The specific random feature functions and distributions have been worked out for many popular kernels. For Gaussian RBF kernel, $k(x-x')=\exp(-\|x - x'\|^2/2\sigma^2)$, this yields a Gaussian distribution $\PP(\omega)$ with density proportional to $\exp(-\sigma^2\|\omega\|^2/2)$; for the Laplace kernel, this yields a Cauchy distribution; and for the Martern kernel, this yields
the convolutions of the unit ball~\cite{SchSmo02}. Similar representation where the explicit form of $\phi_{\omega}(x)$ and $\PP(\omega)$ are known can also be derived for rotation invariant kernel, $k(x,x') = k(\inner{x}{x'})$, using Fourier transformation on sphere~\cite{SchSmo02}. For polynomial kernels, $k(x,x')=(\inner{x}{x'}+c)^p$, a random tensor sketching approach can also be used~\cite{PhaPag13}. See Table~\ref{table:explicit_features} for explicit representations of different kernels.

\section{Algorithm}
\label{sec:algorithm}
In this section, we describe an efficient algorithm based on the ``doubly stochastic gradients'' to scale up kernel PCA. KPCA is essentially an eigenvalue problem in a functional space. Traditional approaches convert it to the dual form, leading to another eigenvalue problem whose size equals the number of training points, which is not scalable.
Other approaches solve it in the primal form with stochastic functional gradient descent. However, these algorithms need to store all the training points seen so far. They quickly run into memory issues when working with hundreds of thousands of data points.

We propose to tackle the problem with ``doubly stochastic gradients'', in which we make two unbiased stochastic approximations. One stochasticity comes from sampling data points as in stochastic gradient descent. Another source of stochasticity is from random features to approximate the kernel.

One technical difficulty in designing doubly stochastic KPCA is an explicit orthogonalization step required in the update rules, which ensures the top $k$ eigenfunctions are orthogonal. This is infeasible for kernel methods on a large dataset since it requires solving an increasingly larger KPCA problem in every iteration. To solve this problem, we formulate the orthogonality constraints into Lagrange multipliers which leads to an Oja-style update rule. The new update enjoys small per iteration complexity and converges to the ground-truth subspace.

We present the algorithm by first deriving the stochastic functional gradient update without random feature approximations, then introducing the doubly stochastic updates.

\subsection{Stochastic functional gradient update}
Kernel PCA can be formulated as the following {\em non-convex} optimization
problem
\begin{align}
\max_{G} \, \tr\rbr{G^\top A G}  \;\; \text{s.t.} \,  G^\top G = I,
\end{align}
where $G:=\rbr{g^1,\ldots,g^k}$ and $g^i$ is the $i$-th function.

The Lagrangian that incorporates the constraint is
\begin{align*}
L(G, \Lambda) =  \tr\rbr{G^\top A G} + \tr\rbr{\rbr{G^\top G - I}\Lambda}
\end{align*}
where $\Lambda$ is the Lagrangian multiplier. The gradient of the Lagrangian w.r.t $G$ is 
\begin{align*}
\nabla_{G}L = 2 A G + G \rbr{\Lambda + \Lambda^\top}.
\end{align*}
Furthermore, from the optimality conditions
\begin{align*}
2 A G + G \rbr{\Lambda + \Lambda^\top} &= 0 , \\
G^\top G - I & = 0,
\end{align*}
we can find $\Lambda + \Lambda^\top = - 2 G^\top A G.$

Plugging this into the gradient, it suggests the following update rule
\begin{align}\label{equ:rule_stoch}
G_{t+1} = G_{t} + \eta_t \rbr{I - G_{t}  G_{t} ^\top } AG_{t} .
\end{align}

Using a stochastic approximation for $A$: $A_t f(\cdot) =  f(x_t)\, k(x_t, \cdot)$, we have $A_t G_t = k(x_t,\cdot) g_t^\top$ and $G_t^\top A_t G_t = g_t g_t^\top $, where $g_t = \sbr{g^1_t(x_t), \dots, g^k_t(x_t)}^\top$.
Therefore, the update rule is
\begin{align}\label{eqn:rule_stoch2}
G_{t+1} = G_{t}\rbr{I - \eta_t g_t g_t^\top} + \eta_t k(x_t,\cdot) g_t^\top .
\end{align}
This rule can also be derived using stochastic gradient and Oja's rule~\cite{Oja82, Oja83}.

\subsection{Doubly stochastic update}
The update rule~(\ref{eqn:rule_stoch2}) has a fundamental computational drawback.
At each time step $t$, a new basis $k(x_t, \cdot)$ is added to $G_t$, and it is therefore a linear combination of the feature mappings of all the data points up to $t$.
This requires the algorithm to store all the data points it has seen so far, which is impractical for large scale datasets.

To address this issue, we use the random feature approximation $k(x, \cdot)  \approx  \phi_{\omega_i}(x) \phi_{\omega_i}(\cdot)$.
Denote $H_t$ the function we get at iteration $t$, the update rule becomes
\begin{align} \label{eqn:rule_doubly}
H_{t+1} = H_{t}\rbr{I - \eta_t h_t {h_t}^\top} + \eta_t \phi_{\omega_t}(x_t) \phi_{\omega_t}(\cdot) {h_t}^\top ,
\end{align}
where $h_t$ is the evaluation of $H_t$ at the current data point: $h_t = \sbr{h^1_t(x_t), \dots, h^k_t(x_t)}^\top$. 

Given $H_0=V_0$, we can explicitly represent $H_t$ as a linear combination of all the random feature functions $\phi_{\omega_i}(\cdot)$:
\begin{align*}
H_t  = \sum_i \phi_{\omega_i}(\cdot) \alpha_i^\top + V_0 \beta,
\end{align*}
where $\alpha_i \in \RR^k$ are the coefficients, and $\beta = \prod_{i\leq t}\rbr{I - \eta_i h_i {h_i}^\top}$.

The update rule on the functions corresponds to the following update for the coefficients
\begin{align*}
\alpha_{t+1} &=  \eta_t \phi_{\omega_t}(x_t) {h_t} \\
\alpha_i &= \alpha_i - \eta_t  \alpha_i^\top h_t  h_t,  \,\,\,\, \forall i \le t
\end{align*}

The specific updates in terms of the coefficients are summarized in Algorithms~1~and~2. Note that in theory new random features are drawn in each iteration, but in practice one can revisit these random features.

\begin{figure}[t]
  \hrule\vspace{1mm}
  \text{\bf Algorithm 1: $\cbr{\alpha_i}_1^t = \text{\bf{DSGD-KPCA}}(\PP(x), k)$}\vspace{1mm}
  \hrule\vspace{1mm}
  \text{\bf Require: $\PP(\omega),\, \phi_{\omega}(x) .$}\\[-4mm]
  \begin{algorithmic}[1]  \label{alg:kpca}
    \FOR{$i=1,\ldots, t$}
      \STATE Sample $x_i \sim \PP(x)$.
      \STATE Sample $\omega_i \sim \PP(\omega)$ with {\color{red}seed $i$}.
      \STATE $h_i = \text{\bf Evaluate}(x_i,\cbr{\alpha_j}_{j=1}^{i-1}) \, \in \RR^k$.
      \STATE $\alpha_i = \eta_i \phi_{\omega_i}(x_i) {h_i}$.
      \STATE $\alpha_j = \alpha_j - \eta_i  \alpha_j^\top h_i  h_i,$ for $j =1,\ldots,i-1$.
    \ENDFOR
  \end{algorithmic}
  \hrule
\end{figure}

\begin{figure}[t]
  \hrule\vspace{1mm}
  \text{\bf Algorithm 2: $h = \text{Evaluate}(x,\,\cbr{\alpha_i}_{i=1}^t)$}\vspace{1mm}
  \hrule\vspace{1mm}
  \text{\bf Require: $\PP(\omega),\, \phi_{\omega}(x).$}\\[-4mm]
  \begin{algorithmic}[1] \label{alg:evaluate}
    \STATE Set $h = 0 \in \RR^k$.
    \FOR{$i=1,\ldots, t$}
      \STATE Sample $\omega_i \sim \PP(\omega)$ with {\color{red}seed $i$}.
      \STATE $h = h +  \phi_{\omega_i}(x) \alpha_i $.
    \ENDFOR
  \end{algorithmic}
  \hrule
\end{figure}

\section{Analysis}
\label{sec:analysis}
In this section, we provide finite time convergence guarantees for our algorithm. 
As discussed in the previous section, explicit orthogonalization is not scalable for the kernel case, therefore we need to provide guarantees for the updates without orthogonalization. This challenge is even more prominent when using random features, since it introduces additional variance.

Furthermore, our guarantees are w.r.t. the top $k$-dimension subspace. Although the convergence without normalization for a top eigenvector has been established before~\cite{Oja82,Oja83}, the subspace case is complicated by the fact that there are $k$ angles between $k$-dimension subspaces, and we need to bound the \emph{largest} angle. To the best of our knowledge, our result is the first finite time convergence result for a subspace \emph{without} explicit orthogonalization.

Note that even though it appears our algorithm is similar to~\cite{DaiXieHe14} on the surface, the underlying analysis is fundamentally different. In~\cite{DaiXieHe14}, the result only applies to \emph{convex problems} where every local optimum is a global optimum while the problems we consider are highly \emph{non-convex}. As a result, many techniques that~\cite{DaiXieHe14} builds upon are not applicable.

\subsection{Notations}
In order to analyze the convergence of our doubly stochastic kernel PCA algorithm, 
we will need to define a few intermediate subspaces. For simplicity of notation, we will assume the mini-batch size for the data points is one. 
\begin{enumerate}[noitemsep]
  \item Let $F_{t}:=\rbr{f_t^1,\ldots,f_t^k}$ be the subspace estimated using stochastic gradient and explicit orthogonalization:
  \begin{align*}  
      &\tilde{F}_{t+1} \leftarrow F_t + \eta_t  A_t F_t \\
      &F_{t+1} \leftarrow \tilde{F}_{t+1}  \rbr{\tilde{F}_{t+1}^\top \tilde{F}_{t+1}  }^{-1/2} 
  \end{align*}  
  
  \item Let $G_{t}:=\rbr{g_t^1,\ldots,g_t^k}$ be the subspace estimated using stochastic update rule without orthogonalization:
  \begin{align*}
      G_{t+1} \leftarrow G_t + \eta_t\rbr{I - G_t G_t^\top}  A_t G_t.
  \end{align*}
  where $A_tG_t$ and $G_t G_t^\top A_t G_t$ can be equivalently written using the evaluation of the function $\cbr{g_t^i}$  on the current data point, leading to the equivalent rule :
\begin{align}
G_{t+1} \leftarrow G_{t}\rbr{I - \eta_t g_t g_t^\top} + \eta_t k(x_t,\cdot) g_t^\top .
\end{align}

  \item Let $\Gtil_{t}:=\rbr{\gtil_t^1,\ldots,\gtil_t^k}$ be the subspace estimated using stochastic update rule without orthogonalization, but the evaluation of the function $\cbr{\gtil_t^i}$ on the current data point is replaced by the evaluation $h_t = \sbr{h_t^i(x_t)}^\top$:
  \begin{align*} 
    \Gtil_{t+1} \leftarrow \Gtil_{t} + \eta_t k(x_t,\cdot) h_t^\top - \eta_t \Gtil_t h_t h_t^\top
  \end{align*}

  \item Let $H_{t}:=\rbr{h_t^1,\ldots,h_t^k}$ be the subspace estimated using doubly stochastic update rule without orthogonalization, \ie, the update rule:
  \begin{align} \label{eqn:rule_doubly}
    H_{t+1} \leftarrow H_{t} + \eta_t \phi_{\omega_t}(x_t) \phi_{\omega_t}(\cdot) h_t^\top - \eta_t H_t h_t h_t^\top. 
  \end{align}
\end{enumerate}

The relation of these subspaces are summarized in Table~\ref{tb:relation}. 
Using these notations, we describe a sketch of our analysis in the rest of the section, while the complete proofs are provided in the appendix. 

We first consider the subspace $G_t$ estimated using the stochastic update rule, since it is simpler and its proof can provide the bases for analyzing the subspace $H_t$ estimated by the doubly stochastic update rule. 

\begin{table}[h!]
  \centering
  \setlength{\tabcolsep}{2pt}
  \caption{Relation between various subspaces.
  \label{tb:relation}
	}
  \vspace{-2mm}
  \begin{tabular}{c|c|c|c|c}
    \hline 
    \hline
    Subspace & Evaluation & Orth. & Data Mini-batch & RF Mini-batch \\
    \hline 
    $V$ & -- & -- & -- & -- \\ 
    $F_t$ & $f_t(x)$ & \cmark & \cmark & \xmark  \\
    $G_t$ & $g_t(x)$ & \xmark & \cmark & \xmark \\
    $\Gtil_t$ & $\gtil_t(x)$ & \xmark & \cmark & \xmark \\
    $H_t$ & $h_t(x)$ & \xmark & \cmark & \cmark \\
    \hline
    \hline
  \end{tabular}
\end{table}

\subsection{Conditions and Assumptions}
We will focus on the case when a good initialization $V_0$ is given:
\begin{align}\label{eqn:init}
	V_0^\top V_0 = I, ~~\cos^2 \theta(V, V_0) \geq 1/2.
\end{align}
In other words, we analyze the later stage of the convergence, which is typical in the literature (\eg, \cite{Shamir14}).
The early stage can be analyzed using established techniques (\eg, \cite{BalDasFre13}).

Throughout the paper we suppose $\abr{k(x,x')} \leq \kappa, \abr{\phi_\omega(x)} \leq \phi$ and regard $\kappa$ and $\phi$ as constants. Note that this is true for all the kernels and corresponding random features considered. We further regard the eigengap $\lambda_k - \lambda_{k+1}$ as a constant, which is also true for typical applications and datasets.

\subsection{Update without random features}

\label{sec:analysis_no_random}
Our guarantee is on the cosine of the principal angle between the computed subspace and the ground truth eigen subspace (also called potential function): $\cos^2 \theta(V,G_t) = \min_w \frac{\nbr{V^\top G_t w}^2}{\nbr{G_t w}^2}$.

Consider the two different update rules, one with explicit orthogonalization and another without
\begin{align*} 
F_{t+1} &\leftarrow \textbf{orth}(F_t + \eta_t  A_t F_t) \\ 
G_{t+1} &\leftarrow G_t + \eta_t\rbr{I - G_t G_t^\top}  A_t G_t 
\end{align*}
where $A_t$ is  the empirical covariance of a mini-batch. 
Our final guarantee for $G_t$ is the following.
\begin{theorem} \label{thm:stoch}
Assume (\ref{eqn:init}) and suppose the mini-batch sizes satisfy that for any $1\leq i \leq t$,
$\nbr{A - A_i} < (\lambda_k - \lambda_{k+1})/8.$
There exist step sizes $\eta_i = O(1/i)$ such that 
\[
  1-\cos^2\theta(V, G_t) = O(1/t).
\]
\end{theorem}

The convergence rate $O(1/t)$ is in the same order as that of computing only the top eigenvector in linear PCA~\cite{BalDasFre13}. The bound requires the mini-batch size is large enough so that the spectral norm of $A$ is approximated up to the order of the eigengap. This is because the increase of the potential is in the order of the eigengap. Similar terms appear in the analysis of the noisy power method~\cite{HarPri14} which, however, requires orthogonalization and is not suitable for the kernel case. We do not specify the mini-batch size, but by assuming suitable data distributions, it is possible to obtain explicit bounds; see for example~\cite{Vershynin12,CaiHar12}.

\noindent 
{\bf Proof sketch} 
We first prove the guarantee for the orthogonalized subspace $F_t$ which is more convenient to analyze, and then show that the updates for $F_t$ and $G_t$ are first order equivalent so $G_t$ enjoys the same guarantee. To do so, we will require lemma~\ref{lem:Ft} and~\ref{lem:equivalence} below
\begin{lemma} \label{lem:Ft}
  $1-\cos^2\theta(V, F_t) = O(1/t)$.
\end{lemma}
Let $c_t^2$ denote $\cos^2\theta(V, F_t)$, 
then a key step in proving the lemma is to show the following recurrence
\begin{align} \label{eqn:recurrence}
	c_{t+1}^2 \geq c_t^2( 1 + 2\eta_t (\lambda_k - \lambda_{k+1} - 2 \nbr{A - A_t}) (1-c_t^2) ) - O(\eta_t^2).
\end{align}
We will need the mini-batch size large enough so that $2\nbr{A - A_t}$ is smaller than the eigen-gap.

Another key element in the proof of the theorem is the first order equivalence of the two update rules. To show this, we introduce $F(G_t) \leftarrow \textbf{orth}(G_t + \eta_t A_t G_t)$ to denote the subspace by applying the update rule of $F_t$ on $G_t$.
We show that the potentials of $G_{t+1}$ and $F(G_t)$ are close:
\begin{lemma} \label{lem:equivalence}
$
	\cos^2\theta(V, G_{t+1}) = \cos^2\theta(V, F(G_t)) \pm O(\eta_t^2).
$
\end{lemma}
The lemma means that applying the two update rules to the same input will result in two subspaces with similar potentials. Then by (\ref{eqn:recurrence}),  we have $1-\cos^2\theta(V, G_t) = O(1/t)$ which leads to our theorem.
The proof of Lemma~\ref{lem:equivalence} is based on the observation that $\cos^2 \theta(V, X)  =  \lambda_{\text{min}}(V^\top  X(X^\top X)^{-1} X^\top V)$.
Comparing the Taylor expansions w.r.t.\ $\eta_t$  for $X=G_{t+1}$ and $X=F(G_t)$ leads to the lemma.

\subsection{Doubly stochastic update} 

The $H_t$ computed in the doubly stochastic update is no longer in the RKHS so the principal angle is not well defined. Instead, we will compare the evaluation of functions from $H_t$ and the true principal subspace $V$ respectively on a point $x$. Formally, we show that for any function $v\in V$ with unit norm $\nbr{v}_\Fcal=1$, there exists a function $h$ in $H_t$ such that for any $x$, $\text{err}:=\abr{v(x) - h(x)}^2$ is small with high probability. 

To do so, we need to introduce a companion update rule: $\Gtil_{t+1} \leftarrow \Gtil_{t} + \eta_t k(x_t,\cdot) h_t^\top - \eta_t \Gtil_t h_t h_t^\top$ resulting in function in the RKHS, but the update makes use of function values from $h_t\in H_t$ which outside the RKHS.  
Let $w = \Gtil^\top v$ be the coefficients of $v$ projected onto $\Gtil$, $h = H_t w$, and $z = \Gtil_t w$. 
Then the error can be decomposed as 
\begin{align}
  \abr{v(x) - h(x)}^2 & = \abr{v(x) - z(x) + z(x) - h(x)}^2 \leq 2\abr{v(x) - z(x)}^2 + 2\abr{z(x) - h(x)}^2 \nonumber\\
  & \leq \underbrace{2 \kappa^2 \nbr{v - z}^2_{\Fcal}}_{\text{(I: Lemma~\ref{lem:RKHSbound})}}
  + \underbrace{2\abr{z(x) - h(x)}^2}_{\text{(II: Lemma~\ref{lem:bound})}}. \label{eqn:error}
\end{align}
%
By definition,  $\nbr{v - z}_{\Fcal}^2 = \nbr{v}_\Fcal^2 - \nbr{z}_\Fcal^2  \leq 1 - \cos^2 \theta(V, \Gtil_t)$,
so the first error term can be bounded by the guarantee on $\Gtil_t$, which can be obtained by similar arguments in Theorem~\ref{thm:stoch}. For the second term, note that $\Gtil_t$ is defined in such a way that the difference between $z(x)$ and $h(x)$ is a martingale, which can be bounded by careful analysis.
\begin{theorem}\label{thm:doubly}
Assume (\ref{eqn:init}) and suppose the mini-batch sizes satisfy that for any $1\leq i \leq t$, $\nbr{A - A_i} < (\lambda_k - \lambda_{k+1})/8$ and are of order $\Omega(\ln \frac{t}{\delta})$.
There exist step sizes $\eta_i = O(1/i)$, such that the following holds. If $\Omega(1) = \lambda_k(\Gtil_i^\top \Gtil_i) \leq \lambda_1(\Gtil_i^\top \Gtil_i) = O(1)$ for all $1\leq i\leq t$, then for any $x$ and any function $v$ in the span of $V$ with unit norm $\nbr{v}_\Fcal = 1$, we have that with probability at least $1- \delta$, there exists $h$ in the span of $H_t$ satisfying
$
  |v(x) - h(x)|^2 = O\rbr{\frac{1}{t} \ln \frac{t}{\delta}}.
$
\end{theorem}
The point-wise error scales as $\Otil(1/t)$ with the step $t$.  Besides the condition that $\nbr{A - A_i}$ is up to the order of the eigengap, we additionally need that the random features approximate the kernel function up to constant accuracy on all the data points up to time $t$, which eventually leads to $\Omega(\ln \frac{t}{\delta})$ mini-batch sizes. Finally, we need $\Gtil_i^\top \Gtil_i$ to be roughly isotropic, \ie, $\Gtil_i$ is roughly orthonormal. Intuitively, this should be true for the following reasons: $\Gtil_0$ is orthonormal; the update for $\Gtil_t$ is close to that for $G_t$, which in turn is close to $F_t$ that are orthonormal. 

\noindent 
{\bf Proof sketch} 
In order to bound term I in (\ref{eqn:error}), we show that  
\begin{lemma} 
  \label{lem:RKHSbound} 
  $1-\cos^2\theta(V, \Gtil_t) = O\rbr{\frac{1}{t} \ln \frac{t}{\delta}}$.
\end{lemma}
This is proved by following similar arguments to get the recurrence~(\ref{eqn:recurrence}), except with an additional error term, which is caused by the fact that the update rule for $\Gtil_{t+1}$ is using the evaluation $h_t(x_t)$ rather than $\gtil_t(x_t)$. Bounding this additional term thus relies on bounding the difference between $h_t(x) - \gtil_t(x)$, which is also what we need for bounding term II in (\ref{eqn:error}). For this, we show:
\begin{lemma} 
  \label{lem:bound} 
	For any  $x$ and unit vector $w$, with probability $\geq 1-\delta$ over $(\Dcal^t,\omega^t)$, 
$
  |\gtil_t(x)w - h_t(x) w|^2= O\rbr{\frac{1}{t}\ln\rbr{\frac{t}{\delta}}}.
$
\end{lemma}
The key to prove this lemma is that our construction of $\Gtil_t$ makes sure that the difference between $\gtil_t(x)w$ and $h_t(x)w$ consists of their difference in each time step. Furthermore, the difference forms a martingale and thus can be bounded by Azuma's inequality. See the supplementary for the details.

\section{Extensions} 
\label{sec:extensions}
The proposed algorithm is a general technique for solving eigenvalue problems in the functional space. Numerous machine learning algorithms boil down to this fundamental operation. Therefore, our method can be easily extended to solve many related tasks, including latent variable estimation, kernel CCA, spectral clustering, \etc.

We briefly illustrate how to extend to different machine learning algorithms in the following subsections.

\subsection{Locating individual eigenfunctions}
The proposed algorithm finds the subspace spanned by the top $k$ eigenfunctions, but it does not isolate the individual eigenfunctions. When we need to locate these individual eigenfunctions, we can use a modified version, called Generalized Hebbian Algorithm (GHA) \cite{Sanger89b}. Its update rule is
\begin{align}
G_{t+1} = G_{t} + \eta_t A_tG_{t}  - \eta_t    G_{t} \operatorname{UT}\sbr{ G_{t} ^\top A_tG_{t}},
\end{align}
where $\operatorname{UT}\sbr{\cdot}$ is an operator that sets the lower triangular parts to zero.

To understand the effect of the upper triangular operator, we can see that $\operatorname{UT}\sbr{\cdot}$ forces the update rule for the first function of $G_t$ to be exactly the same as that of one-dimensional subspace; all the contributions from the other functions are zeroed out. 
\begin{align}
g_{t+1}^1 = g_{t}^1 + \eta_t A_t g_{t}^1  - \eta_t    g_{t}^1 { g_{t}^1} ^\top A_t g^1_{t},
\end{align}

Therefore, the first function will converge to the eigenfunction corresponding to the top eigenvalue.

For all the other functions, $\operatorname{UT}\sbr{\cdot}$ implements a Gram-Schmidt-like orthogonalization that subtracts the contributions from other eigenfunctions.

\subsection{Latent variable models  and kernel SVD}
Latent variable models are probabilistic models that assume unobserved or latent structures in the data. It appears in specific forms such as Gaussian Mixture Models (GMM), Hidden Markov Models (HMM) and Latent Dirichlet Allocations (LDA), \etc. 

The EM algorithm \cite{DemLaiRub77} is considered the standard approach to solve such models. Recently, spectral methods have been proposed to estimate latent variable models with provable guarantees \cite{AnaFosHsuKakLiu12,SonAnaDaiXie14}.
Compared with the EM algorithm , spectral methods are faster to compute and do not suffer from local optima.

\begin{figure}[t]
  \hrule\vspace{1mm}
  \text{\bf Algorithm 3: $\cbr{ \alpha_i, \beta_i }_1^t = \text{\bf{DSGD-KSVD}}(\PP(x), \PP(y), k)$}\vspace{1mm}
  \hrule\vspace{1mm}
  \text{\bf Require: $\PP(\omega),\, \phi_{\omega}(x) .$}\\[-4mm]
  \begin{algorithmic}[1]  \label{alg:kpca}
    \FOR{$i=1,\ldots, t$}
      \STATE Sample $x_i \sim \PP(x)$. Sample $y_i \sim \PP(y)$.
      \STATE Sample $\omega_i \sim \PP(\omega)$ with {\color{red}seed $i$}.
      \STATE $u_i = \text{\bf Evaluate}(x_i,\cbr{\alpha_j}_{j=1}^{i-1}) \, \in \RR^k$.
      \STATE $v_i = \text{\bf Evaluate}(y_i,\cbr{\beta_j}_{j=1}^{i-1}) \, \in \RR^k$.
      \STATE $W = u_i v_i^\top + v_i u_i^\top$
      \STATE $\alpha_i = \eta_i \phi_{\omega_i}(x_i) {v_i}$.
      \STATE $\beta_i = \eta_i \phi_{\omega_i}(y_i) {u_i}$.
      \STATE $\alpha_j = \alpha_j - \eta_i  W\alpha_j ,$ for $j =1,\ldots,i-1$.
      \STATE $\beta_j = \beta_j - \eta_i  W\beta_j, $ for $j =1,\ldots,i-1$.
    \ENDFOR
  \end{algorithmic}
  \hrule
\end{figure}

The key algorithm behind  spectral methods is the SVD. However, kernel SVD scales quadratically with the number of data points.
Our algorithm can be straightforwardly extended to solve kernel SVD. The extension hinges on the following relation
\begin{align*}
 \left[ \begin{array}{cc}
0 & A^\top\\
A & 0
\end{array} \right]  
\left [\begin{array}{c}
V \\
U
\end{array} \right] 
= 
\left[ \begin{array}{c}
A^\top U\\
A V
\end{array} \right]  
= 
\left[ \begin{array}{c}
 V \\
  U
\end{array} \right] \Sigma,
\end{align*}
where $U\Sigma V^\top$ is the SVD of $A$.

It is therefore reduced to the eigenvalue problem. Plugging it into the update rule and treating the two blocks separately, we thus get  two 
simultaneous update rules
 \begin{align}
 W_t & = U_t^\top AV_t + V_t^\top A^\top U_t \\
  U_{t+1} &= U_{t} + \eta_t \rbr{A V_t  -  U_t W_t  },  \\
V_{t+1} &= V_{t} + \eta_t \rbr{ A^\top U_t -  V_t W_t }.
\end{align} 

The algorithm for updating the coefficients is summarized in Algorithm~3.

\subsection{Kernel CCA and generalized eigenvalue problem}
Kernel CCA and ICA \cite{BacJor02} can also be solved under the proposed framework because they can be viewed as generalized eigenvalue problem.

Given two variables $X$ and $Y$, CCA finds two projections such that the correlations between the two projected variables are maximized. 
Given the covariance matrices $C_{XX}$, $C_{YY}$, and $C_{XY}$, CCA is equivalent to the following problem
\begin{align*}
\sbr{
\begin{array}{cc}
C_{XX} & C_{XY} \\
C_{YX} & C_{YY}
\end{array}}
\sbr{\begin{array}{c}
g_X \\
g_Y
\end{array}}
=
\rbr{1 + \sigma^2}
\sbr{
\begin{array}{cc}
C_{XX} & \\
& C_{YY}
\end{array}
}
\sbr{\begin{array}{c}
g_X \\
g_Y
\end{array}},
\end{align*}
where $g_X$ and $g_Y$ are the top canonical correlation functions for variables $X$ and $Y$, respectively, and $\sigma$ is the corresponding canonical correlation.

This is a generalized eigenvalue problem. It can reformulated as the following non-convex optimization problem
\begin{align}
\max_{G} \tr \left( G^\top A G \right) ,\\
\text{s.t.} \;\; G^\top B G = I .
\end{align}

Following the derivation for the standard eigenvalue problem, we get the foliowing update rules
\begin{align}
G_{t+1} = G_t + \eta_t \rbr{ I -  B G_t G_t^\top }A G_t.
\end{align}

Denote $G_t^X$ and $G_t^Y$ the canonical correlation functions for $X$ and $Y$, respectively. We can 
rewrite the above update rule as two simultaneous rules
\begin{align}
W_t &=  { {G_t^Y}^\top C_{YX}G_t^X + {G_t^X}^\top C_{XY} G_t^Y } \\
G_{t+1}^X &= G_t^X + \eta_t \sbr{ C_{XY}G_t^Y - C_{XX}G_t^X  W }\\
G_{t+1}^Y &= G_t^Y + \eta_t \sbr{C_{YX}G_t^X - C_{YY}G_t^Y  W}.
\end{align}
We present the detailed updates for coefficients in Algorithm~4.

\begin{figure}[t]
  \hrule\vspace{1mm}
  \text{\bf Algorithm 4: $\cbr{ \alpha_i, \beta_i }_1^t = \text{\bf{DSGD-KCCA}}(\PP(x), \PP(y), k)$}\vspace{1mm}
  \hrule\vspace{1mm}
  \text{\bf Require: $\PP(\omega),\, \phi_{\omega}(x) .$}\\[-4mm]
  \begin{algorithmic}[1]  \label{alg:kpca}
    \FOR{$i=1,\ldots, t$}
      \STATE Sample $x_i \sim \PP(x)$. Sample $y_i \sim \PP(y)$.
      \STATE Sample $\omega_i \sim \PP(\omega)$ with {\color{red}seed $i$}.
      \STATE $u_i = \text{\bf Evaluate}(x_i,\cbr{\alpha_j}_{j=1}^{i-1}) \, \in \RR^k$.
      \STATE $v_i = \text{\bf Evaluate}(y_i,\cbr{\beta_j}_{j=1}^{i-1}) \, \in \RR^k$.
      \STATE $W = u_i v_i^\top + v_i u_i^\top$
      \STATE $\alpha_i = \eta_i  \phi_{\omega_i}(x_i) \sbr{ v_i -  W u_i }$.
      \STATE $\beta_i = \eta_i  \phi_{\omega_i}(y_i) \sbr{ u_i -  W v_i }$.
    \ENDFOR
  \end{algorithmic}
  \hrule
\end{figure}

\subsection{Kernel sliced inverse regression}
Kernel sliced inverse regression~\cite{KimPav09} aims to do sufficient dimension reduction in which the found low dimension representation preserves the statistical correlation with the targets. It also reduces to a generalized eigenvalue problem, and has been shown to find the same subspace as KCCA~\cite{KimPav09}.

\begin{figure}
\centering{
\includegraphics[width=0.45\columnwidth]{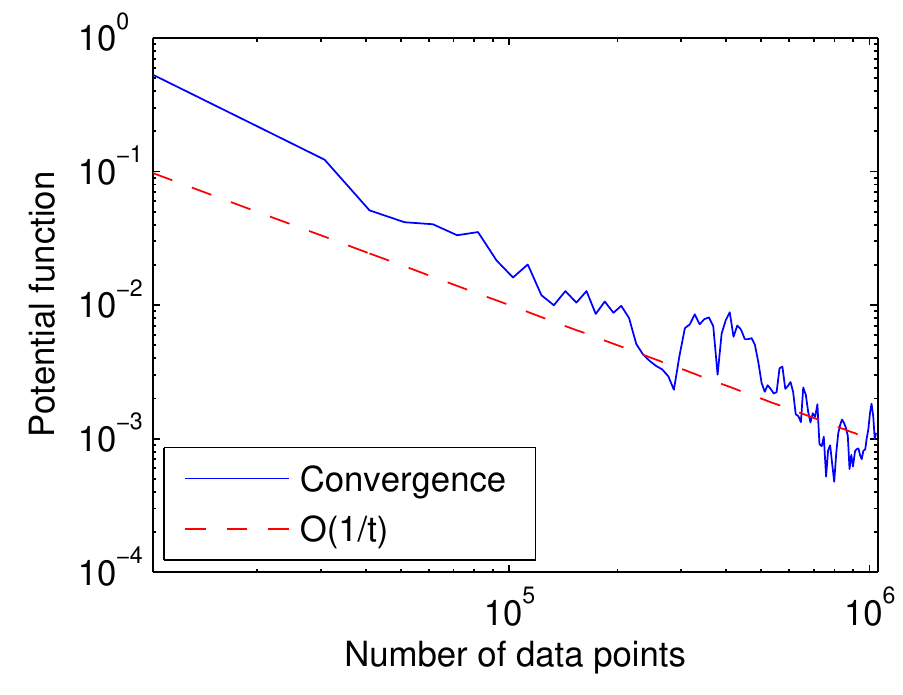}
\caption{Convergence for DSGD-KPCA on the dataset with analytical solution.}
\label{fig:converge_analytical}
}
\end{figure}

\section{Experiments}
\label{sec:experiments}

We demonstrate the effectiveness and scalability of our algorithm on both synthetic and real world datasets.

\subsection{Synthetic dataset with analytical solution}
We first verify the convergence rate of DSGD-KPCA on a synthetic dataset with analytical solution of eigenfunctions \cite{WilSee00}. If the data follow a Gaussian distribution, and we use a Gaussian kernel, then the eigenfunctions are given by the Hermite polynomials.

We generated 1 million data points, and ran DSGD-KPCA with a total of 262,144 random features. In each iteration, we use a data mini-batch of size 512, and a random feature mini-batch of size 128. After all random features are generated, we revisit and adjust the coefficients of existing random features. The kernel bandwidth is set as the true bandwidth of the data.

The step size is scheduled as
\begin{align}
\eta_t = \frac{\theta_0}{1 + \theta_1 t},
\end{align}
where $\theta_0$ and $\theta_1$ are two parameters. We use a small $\theta_1 \approx 0.01$ such that in early stages the step size is large enough to arrive at a good initial solution.

\textbf{Convergence} 
Figure~\ref{fig:converge_analytical} shows the convergence rate of the proposed algorithm seeking top $k=3$ subspace. The potential function is calculated as the squared sine function of the subspace angle between the current solution and the ground-truth. We can see the algorithm indeed converges at the rate $O(1/t)$.

\textbf{Eigenfunction Recovery}
Figure~\ref{fig:eigen_recover} demonstrate the recovered top $k=3$ eigenfunctions compared with the ground-truth. We can see the found solution coincides with one eigenfunction, and only disagree slightly on two others.

\begin{figure}
\centering{
\includegraphics[width=0.45\columnwidth]{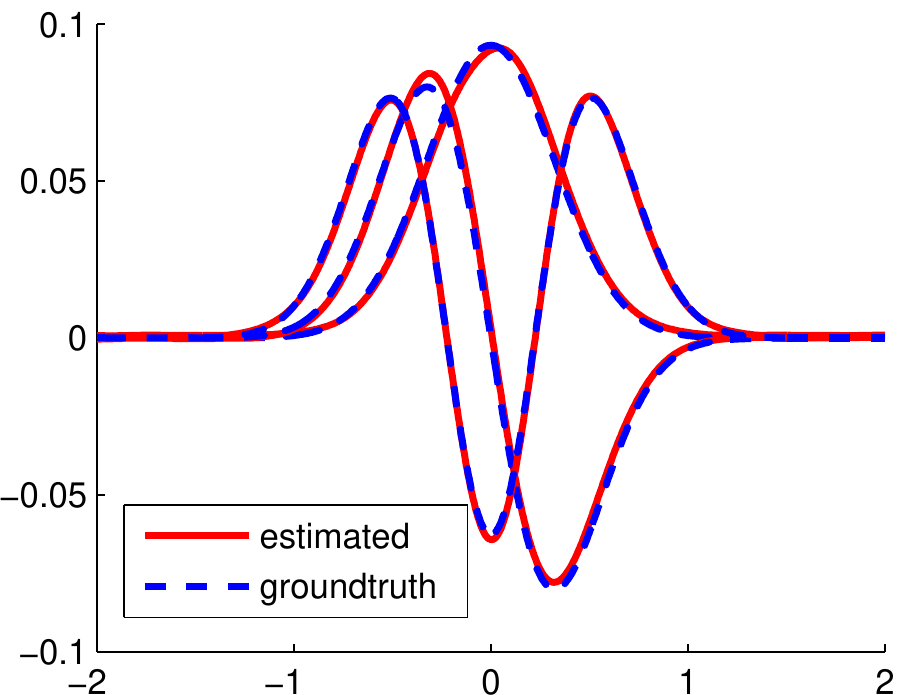}
\caption{Recovered top 3 eigenfunctions using DSGD-KPCA on the dataset with analytical solution.}
\label{fig:eigen_recover}
}
\end{figure}

%
%
%
%
%
%

\subsection{Nonparametric Latent Variable Model}
In \cite{SonAnaDaiXie14}, the authors proposed a multiview nonparametric latent variable model that is solved by kernel SVD followed by tensor power iterations. The algorithm can separate latent variables without imposing specific parametric assumptions of the conditional probabilities. However, the scalability of the algorithm was limited by kernel SVD.

Here, we demonstrate that with DSGD-KSVD, we can learn latent variable models with one million data, achieving higher quality of learned components compared with two other approaches.  

DSGD-KSVD uses a total of 8192 random features, and in each iteration, it uses a feature mini-batch of size 256 and a data mini-batch of size 512.

We compare with 1) random Fourier features with fixed 2048 functions, and 2) random Nystrom features with fixed 2048 functions. The Nystrom features are calculated by first uniformly sampling 2048 data points, and then evaluate kernel function values on these data points \cite{LopSraSmoGhaSch14}.

The dataset consists of two latent components, one is a Gaussian distribution and the other follows a Gamma distribution with shape parameter $\alpha=1.2$. One million data point are generated from this mixture distribution.

Figures~\ref{fig:latent_components} shows the learned conditional distributions for each component. We can see DSGD-KSVD achieves almost perfect recovery, while Fourier and Nystrom random feature methods either confuse high density areas or incorrectly estimate the spread of conditional distributions.

\begin{figure}
\centering{
\includegraphics[width=0.4\columnwidth]{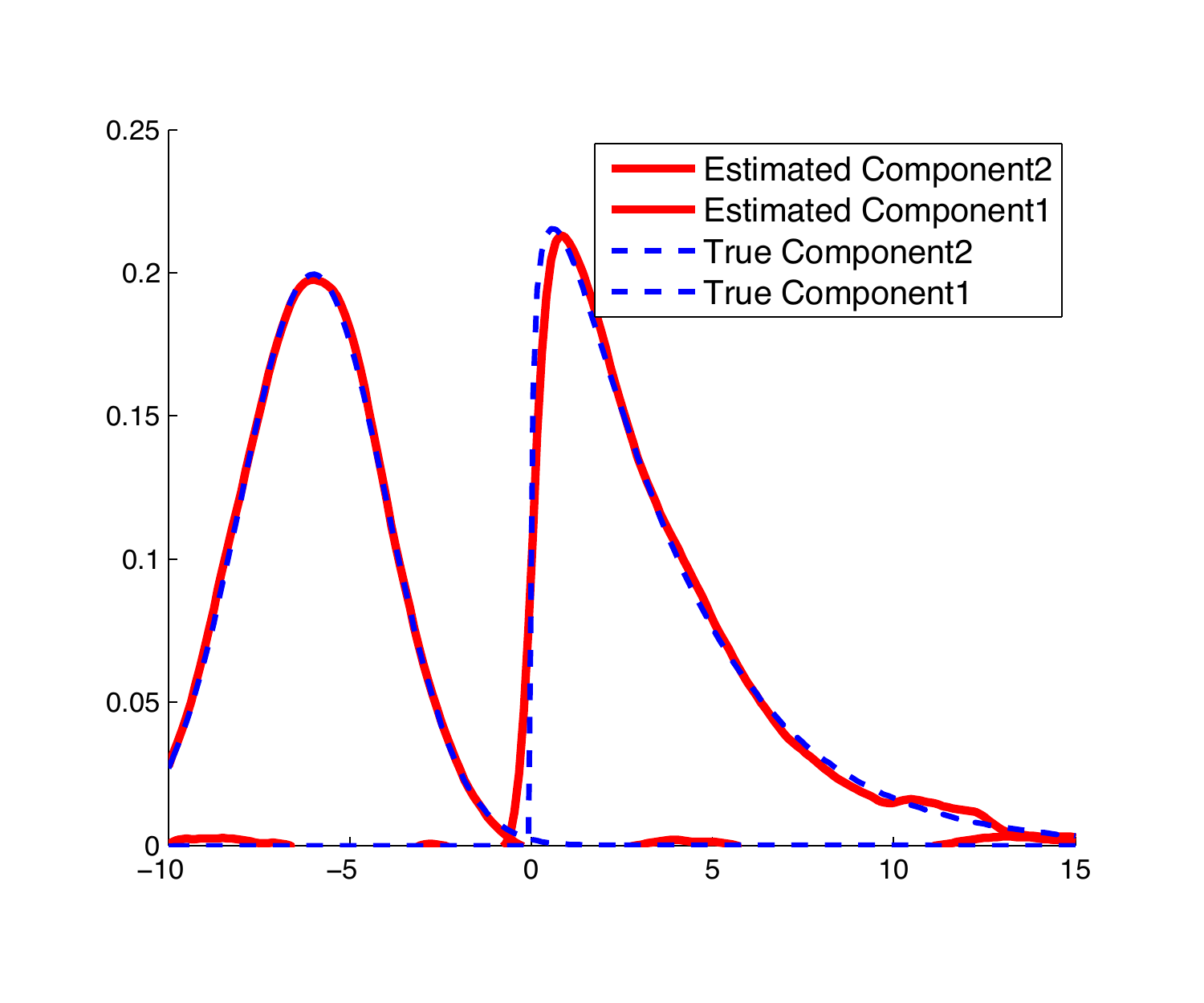}\\
{(a)}\\
\includegraphics[width=0.4\columnwidth]{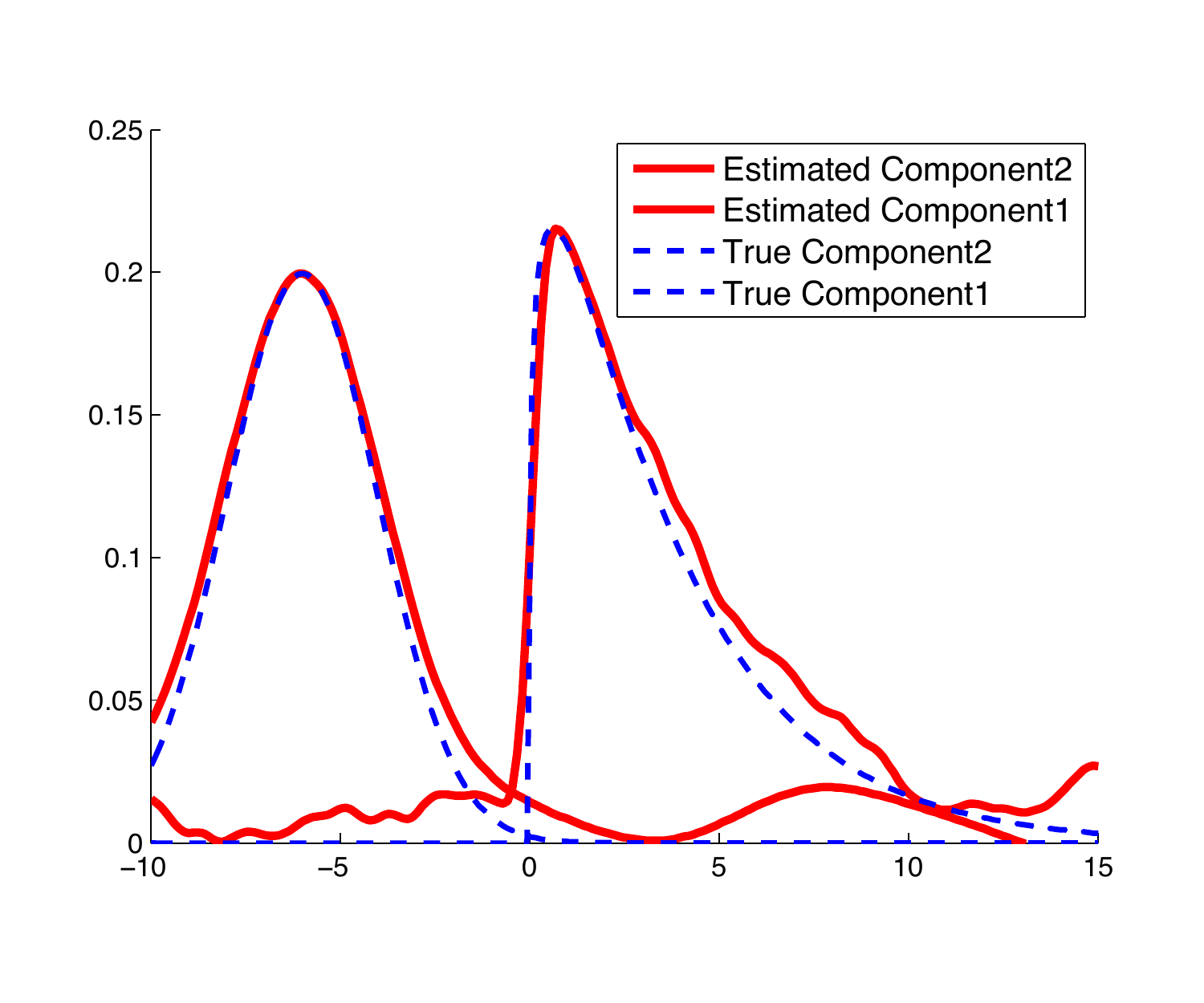}\\
(b)\\
\includegraphics[width=0.4\columnwidth]{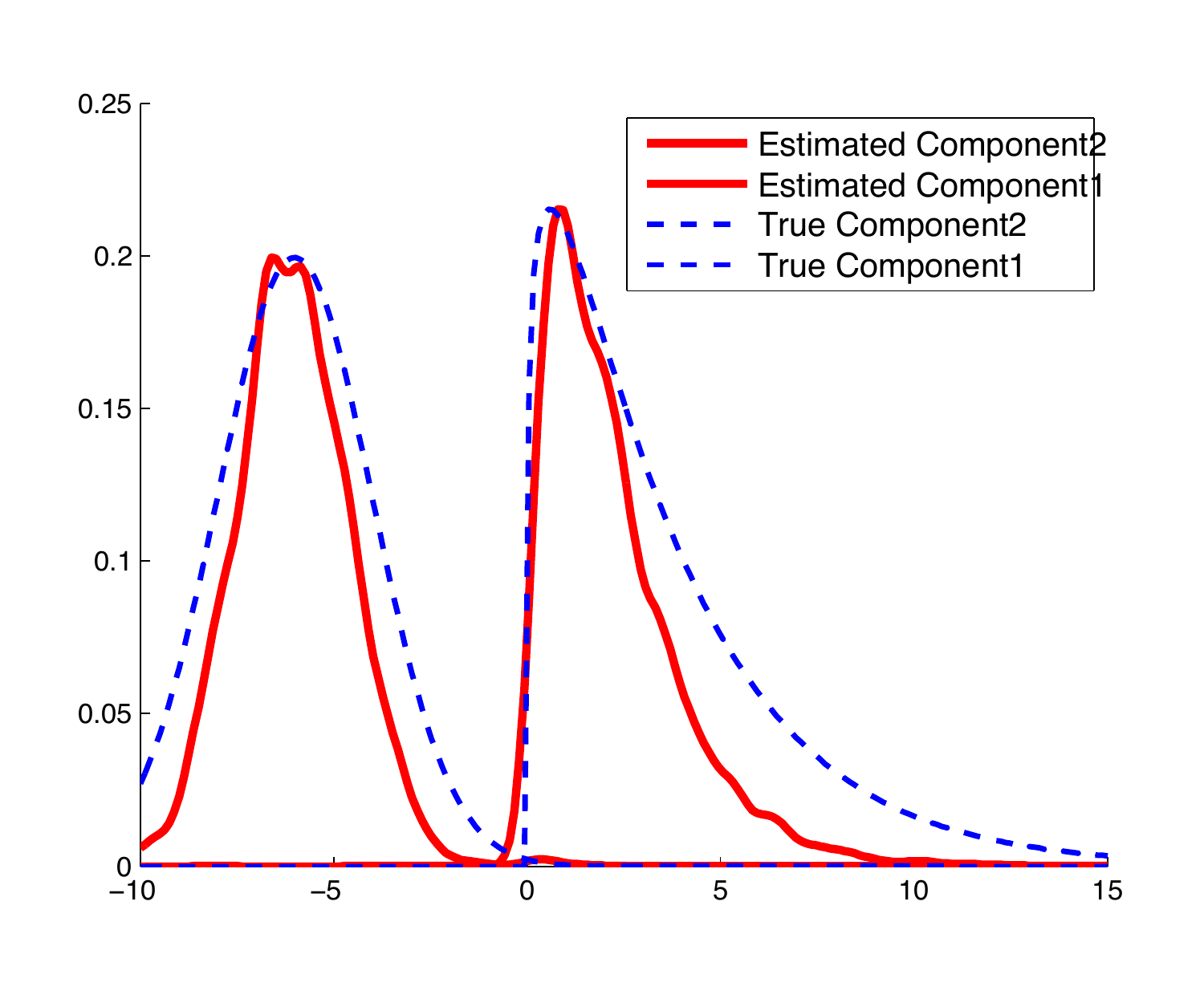}\\
(c)\\
\caption{Recovered latent components (a) DSGD-KSVD, (b) 2048 random features, (c) 2048 Nystrom features.}
\label{fig:latent_components}
}
\end{figure}

\begin{table}
  \vspace{-4pt}
  \setlength{\tabcolsep}{10pt}
  \centering
  \caption{KCCA results on MNIST 8M (top 50 largest correlations)}\label{table:cca_mnist}\vspace{-4pt}
  \begin{tabular}{c|c|c|c|c}
  \hline
 \multirow{2}{*}{ \# of feat} & \multicolumn{2}{c|}{Random features} & \multicolumn{2}{c}{Nystrom features} \\
  \cline{2-5}
   & corrs. & minutes & corrs. & minutes\\
    \hline
  \hline
  256 & 25.2 & 3.2 & 30.4 & 3.0 \\\hline
  512 & 30.7 & 7.0 & 35.3 & 5.1 \\\hline
  1024 & 35.3 & 13.9 & 38.0 & 10.1 \\\hline
  2048 & 38.8 & 54.3 & 41.1 & 27.0 \\\hline
  4096 & 41.5 & 186.7 & 42.7 & 71.0 \\\hline
  \end{tabular}\\
  \bigskip
  \begin{tabular}{c|c|c|c}
  \hline
   \multicolumn{2}{c|}{DSGD-KCCA} & \multicolumn{2}{c}{linear CCA}\\
\hline
corrs. & minutes & corrs. & minutes\\
    \hline
  \hline
  43.5 & 183.2 & 27.4 & 1.1\\\hline
  \end{tabular}
  \vspace{-4pt}
\end{table}

\subsection{KCCA MNIST8M}
We then demonstrate the scalability and effectiveness of our algorithm on a large-scale real world dataset. MNIST8M consists of 8.1 million hand-written digits and their transformations. Each digit is of size $28 \times 28$. We divide each image into the left and right parts, and learn their correlations using KCCA. Thus the input feature dimension is 392. 

The evaluation criteria is the total correlations on the top $k=50$ canonical correlation directions calculated on a separate test set of size 10000. Out of the 8.1 million training data, we randomly choose 10000 as an evaluation set.

We compare with 1) random Fourier and 2) random Nystrom features on both total correlation and running time. We vary the number of random features used for both methods. Our algorithm uses a total of 20480 features. In each iteration, we use feature mini-batches of size 2048 and data mini-batches of size 1024, and we run 3000 iterations. The kernel bandwidth is set using the ``median'' trick and is the same for all methods. Due to randomness, all algorithms are run 5 times, and the mean is reported.

The results are presented in Table~\ref{table:cca_mnist}. We can see Nystrom features generally achieve better results than Fourier features. Note that for Fourier features, we are using the version with $\sin$ and $\cos$ pairs, so the real number of parameters is twice the number in the table, as a result the computational time is almost twice of that for Nystrom features.

Our algorithm achieves the best test-set correlations in comparable run time with random Fourier features. This is especially significant for random Fourier features, since the run time would increase by almost four times if double the number of features were used. We can also see that for large datasets, it is important to use more random features for better performance. Actually, the number of random features required should grow linearly with the number of data points. Therefore, our algorithm provides a good balance between the number of random features used and the number of data points processed.

%

\subsection{Kernel PCA visualization on molecular space dataset}
MolecularSpace dataset contains 2.3 million molecular motifs~\cite{DaiXieHe14}. We are interested in visualizing the dataset with KPCA. The data are represented by sorted Coulomb matrices of size $75 \times 75$ \cite{MonHanFazRupetal12}. Each molecule also has an attribute called power conversion efficiency (PCE). We use a Gaussian kernel with bandwidth chosen by the ``median trick''.
We ran kernel PCA with a total of 16384 random features, with a feature mini-batch size of 512, and data mini-batch size of 1024. We ran 4000 iterations with step size $\eta_t = 1/(1+0.001*t)$.

Figure~\ref{fig:mol_kpca_vis} presents visualization by projecting the data onto the top two principle components.
Compared with linear PCA, KPCA shrinks the distances between the clusters and brings out the important structures in the dataset.
We can also see although the PCE values do not necessarily correspond to the clusters, higher PCE values tend to lie towards the center of the ring structure.

\begin{figure}
\centering{
\begin{subfigure}[b]{0.4\textwidth}
\includegraphics[width=\columnwidth]{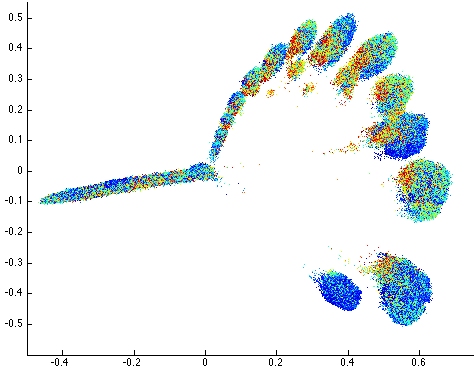}
\caption{}\vspace{-6pt}
\end{subfigure}\hspace{20pt}
\begin{subfigure}[b]{0.4\textwidth}
\includegraphics[width=\columnwidth]{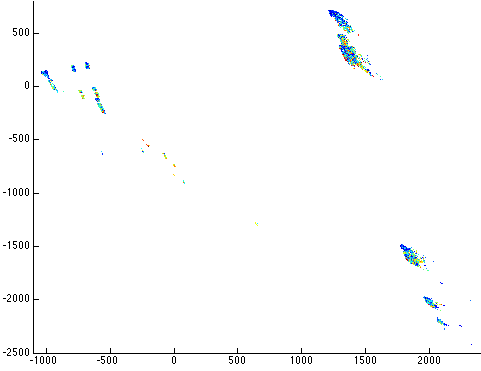}
\caption{}\vspace{-6pt}
\end{subfigure}
\caption{Visualization of the molecular space dataset by the first two principal components. The color corresponds to the PCE value: bluer dots represent lower PCE values while redder dots are for higher PCE values. (a) Kernel PCA; (b) linear PCA. (Best viewed in color)}
\label{fig:mol_kpca_vis}
}
\end{figure}

\begin{figure}
\centering{
\includegraphics[width=0.5\columnwidth]{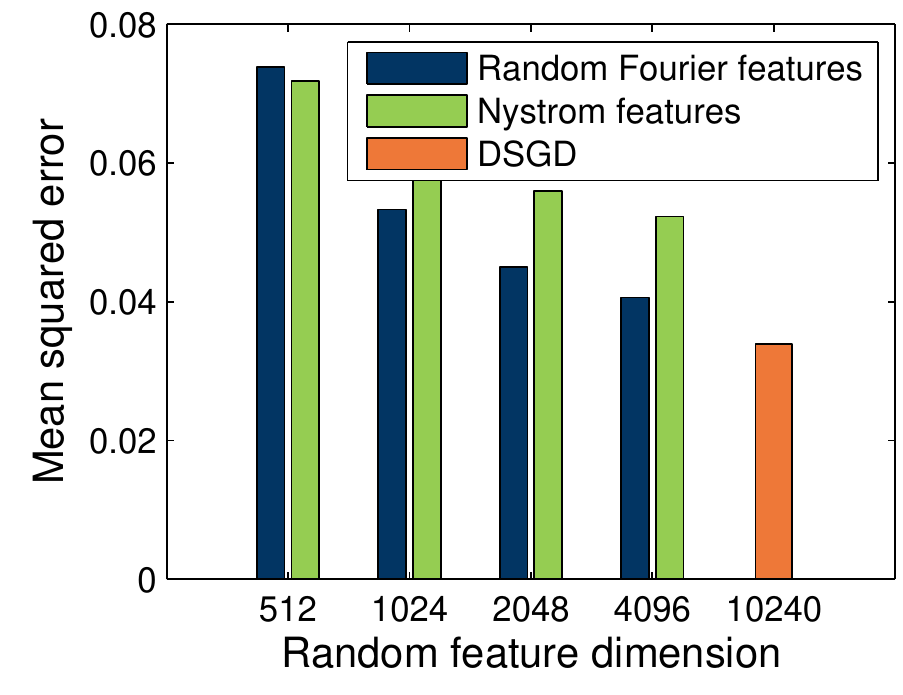}
\caption{Comparison on KUKA dataset.}
\label{fig:kuka_exp}
}
\end{figure}

\subsection{Kernel sliced inverse regression on KUKA dataset}
We evaluate our algorithm under the setting of kernel sliced inverse regression~\cite{KimPav09}, a way to perform sufficient dimension reduction (SDR) for high dimension regression. After performing SDR, we fit a linear regression model using the projected input data, and evaluate mean squared error (MSE).
The dataset records rhythmic motions of a KUKA arm at various speeds, representing realistic settings for robots~\cite{MeiHenSch14}. We use a variant that contains 2 million data points generated by the SL simulator. The KUKA robot has 7 joints, and the high dimension regression problem is to predict the torques from positions, velocities and accelerations of the joints. The input has 21 dimensions while the output is 7 dimensions. Since there are seven independent joints, we set the reduced dimension to be seven. We randomly select 20\% as test set and out of the remaining training set, we randomly choose 5000 as validation set to select step sizes.
The total number of random features is 10240, with mini-feature batch and mini-data batch both equal to 1024. We run a total of 2000 iterations using step size $\eta_t = 15 / (1 + 0.001*t)$.

Figure~\ref{fig:kuka_exp} shows the regression errors for different methods. The error decreases with more random features, and our algorithm achieves lowest MSE by using 10240 random features. Nystrom features do not perform as well in this setting probably because the spectrum decreases slowly (there are seven independent joints) as Nystrom features are known to work well for fast decreasing spectrum.

\section{Conclusions}
We have proposed a general and scalable approach to solve nonlinear component analysis based on doubly stochastic gradients. It is simple, efficient and scalable. In addition, we have theoretical guarantees that the whole subspace converges at the rate $\Otil(1/t)$ to the true subspace. Moreover, since its core is an algorithm for eigenvalue problems in the functional space, it can be applied to various other tasks and models. Finally, we demonstrate the scalability and effectiveness of our algorithm on both synthetic and real world datasets.

\bibliographystyle{abbrv}
  \bibliography{kpca_arxiv}

\begin{thebibliography}{10}

\bibitem{AnaFosHsuKakLiu12}
A.~Anandkumar, D.~P. Foster, D.~Hsu, S.~M. Kakade, and Y.-K. Liu.
\newblock Two svds suffice: Spectral decompositions for probabilistic topic
  modeling and latent dirichlet allocation.
\newblock {\em CoRR}, abs/1204.6703, 2012.

\bibitem{AroCotSre13}
R.~Arora, A.~Cotter, and N.~Srebro.
\newblock Stochastic optimization of pca with capped msg.
\newblock In {\em Advances in Neural Information Processing Systems}, pages
  1815--1823, 2013.

\bibitem{BacJor02}
F.~R. Bach and M.~I. Jordan.
\newblock Kernel independent component analysis.
\newblock {\em Journal of Machine Learning Research}, 3:1--48, 2002.

\bibitem{BalDasFre13}
A.~Balsubramani, S.~Dasgupta, and Y.~Freund.
\newblock The fast convergence of incremental pca.
\newblock In {\em Advances in Neural Information Processing Systems}, pages
  3174--3182, 2013.

\bibitem{CaiHar12}
T.~T. Cai and H.~H. Zhou.
\newblock Optimal rates of convergence for sparse covariance matrix estimation.
\newblock {\em The Annals of Statistics}, 40(5):2389--2420, 2012.

\bibitem{ChiSut07}
T.-J. Chin and D.~Suter.
\newblock Incremental kernel principal component analysis.
\newblock {\em IEEE Transactions on Image Processing}, 16(6):1662--1674, 2007.

\bibitem{ChoSaul09}
Y.~Cho and L.~K. Saul.
\newblock Kernel methods for deep learning.
\newblock In Y.~Bengio, D.~Schuurmans, J.~Lafferty, C.~Williams, and
  A.~Culotta, editors, {\em Advances in Neural Information Processing Systems
  22}, pages 342--350, 2009.

\bibitem{DaiXieHe14}
B.~Dai, B.~Xie, N.~He, Y.~Liang, A.~Raj, M.-F.~F. Balcan, and L.~Song.
\newblock Scalable kernel methods via doubly stochastic gradients.
\newblock In {\em Advances in Neural Information Processing Systems}, pages
  3041--3049, 2014.

\bibitem{DemLaiRub77}
A.~P. Dempster, N.~M. Laird, and D.~B. Rubin.
\newblock Maximum likelihood from incomplete data via the {EM} algorithm.
\newblock {\em Journal of the Royal Statistical Society B}, 39(1):1--22, 1977.

\bibitem{HarPri14}
M.~Hardt and E.~Price.
\newblock The noisy power method: A meta algorithm with applications.
\newblock In {\em Advances in Neural Information Processing Systems}, pages
  2861--2869, 2014.

\bibitem{Honeine12}
P.~Honeine.
\newblock Online kernel principal component analysis: {A} reduced-order model.
\newblock {\em {IEEE} Trans. Pattern Anal. Mach. Intell.}, 34(9):1814--1826,
  2012.

\bibitem{KarKar12}
P.~Kar and H.~Karnick.
\newblock Random feature maps for dot product kernels.
\newblock In N.~D. Lawrence and M.~A. Girolami, editors, {\em AISTATS-12},
  volume~22, pages 583--591, 2012.

\bibitem{KimFraSch05}
K.~Kim, M.~O. Franz, and B.~Sch\"olkopf.
\newblock Iterative kernel principal component analysis for image modeling.
\newblock {\em IEEE Transactions on Pattern Analysis and Machine Intelligence},
  27(9):1351--1366, 2005.

\bibitem{KimPav09}
M.~Kim and V.~Pavlovic.
\newblock Covariance operator based dimensionality reduction with extension to
  semi-supervised settings.
\newblock In {\em International Conference on Artificial Intelligence and
  Statistics}, pages 280--287, 2009.

\bibitem{LeSarSmo13}
Q.~Le, T.~Sarlos, and A.~J. Smola.
\newblock Fastfood --- computing hilbert space expansions in loglinear time.
\newblock In {\em International Conference on Machine Learning}, 2013.

\bibitem{LopSraSmoGhaSch14}
D.~Lopez-Paz, S.~Sra, A.~Smola, Z.~Ghahramani, and B.~Sch{\"o}lkopf.
\newblock Randomized nonlinear component analysis.
\newblock In {\em International Conference on Machine Learning (ICML)}, 2014.

\bibitem{MeiHenSch14}
F.~Meier, P.~Hennig, and S.~Schaal.
\newblock Incremental local gaussian regression.
\newblock In Z.~Ghahramani, M.~Welling, C.~Cortes, N.~Lawrence, and
  K.~Weinberger, editors, {\em Advances in Neural Information Processing
  Systems 27}, pages 972--980. Curran Associates, Inc., 2014.

\bibitem{MonHanFazRupetal12}
G.~Montavon, K.~Hansen, S.~Fazli, M.~Rupp, F.~Biegler, A.~Ziehe, A.~Tkatchenko,
  A.~von Lilienfeld, and K.-R. M{\"u}ller.
\newblock Learning invariant representations of molecules for atomization
  energy prediction.
\newblock In {\em Neural Information Processing Systems}, pages 449--457, 2012.

\bibitem{Oja82}
E.~Oja.
\newblock A simplified neuron model as a principal component analyzer.
\newblock {\em J.~Math.~Biology}, 15:267--273, 1982.

\bibitem{Oja83}
E.~Oja.
\newblock {\em Subspace methods of pattern recognition}.
\newblock John Wiley and Sons, New York, 1983.

\bibitem{PhaPag13}
N.~Pham and R.~Pagh.
\newblock Fast and scalable polynomial kernels via explicit feature maps.
\newblock In {\em Proceedings of the 19th ACM SIGKDD international conference
  on Knowledge discovery and data mining}, pages 239--247. ACM, 2013.

\bibitem{RahRec08}
A.~Rahimi and B.~Recht.
\newblock Random features for large-scale kernel machines.
\newblock In J.~Platt, D.~Koller, Y.~Singer, and S.~Roweis, editors, {\em
  Advances in Neural Information Processing Systems 20}. MIT Press, Cambridge,
  MA, 2008.

\bibitem{RahRec09}
A.~Rahimi and B.~Recht.
\newblock Weighted sums of random kitchen sinks: Replacing minimization with
  randomization in learning.
\newblock In {\em Neural Information Processing Systems}, 2009.

\bibitem{RasWil06}
C.~E. Rasmussen and C.~K.~I. Williams.
\newblock {\em Gaussian Processes for Machine Learning}.
\newblock MIT Press, Cambridge, MA, 2006.

\bibitem{Sanger89b}
T.~D. Sanger.
\newblock Optimal unsupervised learning in a single-layer linear feedforward
  network.
\newblock {\em Neural Networks}, 2:459--473, 1989.

\bibitem{SchSmo02}
B.~Sch{\"o}lkopf and A.~J. Smola.
\newblock {\em Learning with Kernels}.
\newblock {MIT} Press, Cambridge, MA, 2002.

\bibitem{SchGunVis07}
N.~N. Schraudolph, S.~G\"unter, and S.~V.~N. Vishwanathan.
\newblock Fast iterative kernel {PCA}.
\newblock In B.~{Sch\"olkopf}, J.~Platt, and T.~Hofmann, editors, {\em Advances
  in Neural Information Processing Systems 19}, Cambridge MA, June 2007. {MIT}
  Press.

\bibitem{Shamir14}
O.~Shamir.
\newblock A stochastic pca algorithm with an exponential convergence rate.
\newblock {\em arXiv preprint arXiv:1409.2848}, 2014.

\bibitem{SonAnaDaiXie14}
L.~Song, A.~Anamdakumar, B.~Dai, and B.~Xie.
\newblock Nonparametric estimation of multi-view latent variable models.
\newblock In {\em International Conference on Machine Learning (ICML)}, 2014.

\bibitem{Vershynin12}
R.~Vershynin.
\newblock How close is the sample covariance matrix to the actual covariance
  matrix?
\newblock {\em Journal of Theoretical Probability}, 25(3):655--686, 2012.

\bibitem{WilSee00}
C.~K.~I. Williams and M.~Seeger.
\newblock The effect of the input density distribution on kernel-based
  classifiers.
\newblock In P.~Langley, editor, {\em Proc.\ Intl.\ Conf.\ Machine Learning},
  pages 1159--1166, San Francisco, California, 2000. Morgan Kaufmann
  Publishers.

\bibitem{YanSinFanAvretal14}
J.~Yang, V.~Sindhwani, Q.~Fan, H.~Avron, and M.~W. Mahoney.
\newblock Random laplace feature maps for semigroup kernels on histograms.
\newblock In {\em CVPR}, 2014.

\end{thebibliography}

\newpage
\begin{center}
{\Large Appendix}
\end{center}

\begin{appendix}
\label{sec:appendix}

The appendix is organized as follows.
Section~\ref{sec:setting} reviews notations, the definition of Kernel PCA and the update rules considered. Section~\ref{sec:analysis} provides the sketch of the proof as in the paper. Section~\ref{sec:stoch} provides the proof for the stochastic update rule, and Section~\ref{sec:doubly} provides the proof for the doubly stochastic update rule.

\section{Setting} 
\label{sec:setting}

\paragraph{Notations} Given a distribution $\PP(x)$, a kernel function $k(x,x')$ with RKHS $\Fcal$, the covariance operator $A:\Fcal\mapsto \Fcal$ is a linear self-adjoint operator defined as 
\begin{align}
  \label{app_eq:cov}
  A f(\cdot) := \EE_{x}[f(x)\, k(x,\cdot)],\quad \forall f \in \Fcal,  
\end{align}
and furthermore 
\begin{align*}
  \inner{g}{Af}_{\Fcal} = \EE_{x}[f(x)\, g(x)],\quad \forall g \in \Fcal.
\end{align*}

Let $F = \rbr{f_1(\cdot), f_2(\cdot), \dots, f_k(\cdot)}$ be a list of $k$ functions in the RKHS, and we define matrix-like notation
\begin{align}
  A F(\cdot) := \rbr{A f_1(\cdot), \dots, A f_k(\cdot)},   
\end{align}
and $F^\top A F$ is a $k \times k$ matrix, whose $(i, j)$-th element is $\inner{f_i}{A f_j}_{\Fcal}$.
The outer-product of a function $v\in \Fcal$ defines a linear operator $v v^\top:\Fcal\mapsto\Fcal$ such that
\begin{align}
  (v v^\top)f(\cdot)  := \inner{v}{f}_{\Fcal} v(\cdot),\quad \forall f \in \Fcal
\end{align}
Let $V =  \rbr{v_1(\cdot), \dots, v_k(\cdot)}$ be a list of $k$ functions, then 
the weighted sum of a set of linear operators, $\cbr{v_i v_i^\top}_{i=1}^k$, can be denoted using matrix-like notation as
\begin{align}
  V \Sigma_k V^\top := \sum_{i=1}^k \lambda_i v_i v_i^\top 
\end{align}
where $\Sigma_k$ is a diagonal matrix with $\lambda_i$ on the $i$-th entry of the diagonal.

\paragraph{Kernel PCA} Kernel PCA aims to identify the top $k$ eigenfunctions $V =  \rbr{v_1(\cdot), \dots, v_k(\cdot)}$ for the covariance operator $A$, where $V$ is also called the top $k$ subspace for $A$. 

A function $v$ is an eigenfunction of covariance operator $A$ with the corresponding eigenvalue $\lambda$ if 
\begin{align}
  Av(\cdot) = \lambda v(\cdot). 
\end{align}
Given a set of eigenfunctions $\cbr{v_i}$ and associated eigenvalues $\cbr{\lambda_i}$, where $\inner{v_i}{v_j}_{\Fcal} = \delta_{ij}$. We can denote the eigenvalue of $A$ as 
\begin{align}
  A =  V\Sigma_k V^\top + V_{\perp}\Sigma_{\perp} V_{\perp}^\top 
\end{align}
where $V =  \rbr{v_1(\cdot), \dots, v_k(\cdot)}$ is the top $k$ eigenfunctions of $A$, and $\Sigma_k$ is a diagonal matrix with the corresponding eigenvalues, $V_{\perp}$ is the collection of the rest of the eigenfunctions, and $\Sigma_{\perp}$ is a diagonal matrix with the rest of the eigenvalues. 

\paragraph{Update rules} 
The stochastic update rule is
\begin{align}\label{app_equ:rule_stoch}
G_{t+1} = G_{t} + \eta_t \rbr{I - G_{t}  G_{t} ^\top } A_t G_{t} 
\end{align}
where $G_t :=\rbr{g_t^1,\ldots,g_t^k}$ and $g_t^i$ is the $i$-th function.
Denote the evaluation of $G_t$ at the current data point as
\begin{align}
g_t = \sbr{g^1_t(x_t), \dots, g^k_t(x_t)}^\top \in \RR^k.
\end{align}
Then the update rule can be re-written as
\begin{align}\label{app_eqn:rule_stoch2}
G_{t+1} = G_{t}\rbr{I - \eta_t g_t g_t^\top} + \eta_t k(x_t,\cdot) g_t^\top .
\end{align}

The doubly stochastic update rule is
\begin{align} \label{app_eqn:rule_doubly}
H_{t+1} = H_{t}\rbr{I - \eta_t h_t {h_t}^\top} + \eta_t \phi_{\omega_t}(x_t) \phi_{\omega_t}(\cdot) {h_t}^\top ,
\end{align}
where $h_t$ is the evaluation of $H_t$ at the current data point:
\begin{align}
h_t = \sbr{h^1_t(x_t), \dots, h^k_t(x_t)}^\top \in \RR^k.
\end{align}

When larger mini-batch sizes are used, the update rule is adjusted accordingly. For example, when using $B_{x,t}$ points $\cbr{x^b_t}$ and $B_{\omega,t}$ features $\cbr{\omega_t^{b'}}$, the update rule for $H_t$ is
  \begin{align*} 
    H_{t+1} \leftarrow H_{t} 
    &+ \frac{\eta_t  \sum_{b,b'} \rbr{\phi_{\omega^{b'}_t}(x^b_t) \phi_{\omega^{b'}_t}(\cdot) \sbr{h_t^1(x^b_t),\ldots,h_t^k(x^b_t)}}  }{B_{x,t} B_{\omega,t}} \\
    &- \eta_t H_t \rbr{~\frac{1}{B_{x,t}} \sum_{b} \sbr{h_t^i(x^b_t) h_t^j(x^b_t)} ~}_{i,j=1}^k.
  \end{align*}

\section{Analysis Roadmap}
\label{sec:analysis}

In order to analyze the convergence of our doubly stochastic kernel PCA algorithm, 
we will need to define a few intermediate subspaces. For simplicity of notation, we will assume the mini-batch size for the data points is one. 
\begin{enumerate}[noitemsep]
  \item Let $F_{t}:=\rbr{f_t^1,\ldots,f_t^k}$ be the subspace estimated using stochastic gradient and explicit orthogonalization:
  \begin{align}  \label{app_eqn:stoch_orth}
      &\tilde{F}_{t+1} \leftarrow F_t + \eta_t  A_t F_t \\
      &F_{t+1} \leftarrow \tilde{F}_{t+1}  \rbr{\tilde{F}_{t+1}^\top \tilde{F}_{t+1}  }^{-1/2} \nonumber
  \end{align}  
  
  \item Let $G_{t}:=\rbr{g_t^1,\ldots,g_t^k}$ be the subspace estimated using stochastic update rule without orthogonalization:
  \begin{align} \label{app_eqn:stoch}
      G_{t+1} \leftarrow G_t + \eta_t\rbr{I - G_t G_t^\top}  A_t G_t.
  \end{align}
  where $A_tG_t$ and $G_t G_t^\top A_t G_t$ can be equivalently written using the evaluation of the function $\cbr{g_t^i}$  on the current data point, leading to the equivalent rule:
\begin{align}\label{app_eqn:rule_stoch2}
G_{t+1} \leftarrow G_{t}\rbr{I - \eta_t g_t g_t^\top} + \eta_t k(x_t,\cdot) g_t^\top .
\end{align}

  \item Let $\Gtil_{t}:=\rbr{\gtil_t^1,\ldots,\gtil_t^k}$ be the subspace estimated using stochastic update rule without orthogonalization, but the evaluation of the function $\cbr{\gtil_t^i}$ on the current data point is replaced by the evaluation $h_t = \sbr{h_t^i(x_t)}^\top$:
  \begin{align} \label{app_eqn:stoch_rep}
    \Gtil_{t+1} \leftarrow \Gtil_{t} + \eta_t k(x_t,\cdot) h_t^\top - \eta_t \Gtil_t h_t h_t^\top
  \end{align}

  \item Let $H_{t}:=\rbr{h_t^1,\ldots,h_t^k}$ be the subspace estimated using doubly stochastic update rule without orthogonalization, \ie, the update rule:
  \begin{align}  \label{app_eqn:doubly_stoch}
    H_{t+1} \leftarrow H_{t} + \eta_t \phi_{\omega_t}(x_t) \phi_{\omega_t}(\cdot) h_t^\top - \eta_t H_t h_t h_t^\top. 
  \end{align}
\end{enumerate}

The relation of these subspaces are summarized in Table~\ref{app_tb:relation}. 
Using these notations, we describe a sketch of our analysis in the rest of the section, while the complete proofs are provided in the following sections.

We first consider the subspace $G_t$ estimated using the stochastic update rule, since it is simpler and its proof can provide the bases for analyzing the subspace $H_t$ estimated by the doubly stochastic update rule. 

\begin{table}[h!]
  \centering
  \setlength{\tabcolsep}{2pt}
  \caption{Relation between various subspaces.
  \label{app_tb:relation}
	}
  \begin{tabular}{c|c|c|c|c}
    \hline 
    \hline
    Subspace & Evaluation & Orth. & Data Mini-batch & RF Mini-batch \\
    \hline 
    $V$ & -- & -- & -- & -- \\ 
    $F_t$ & $f_t(x)$ & \cmark & \cmark & \xmark  \\
    $G_t$ & $g_t(x)$ & \xmark & \cmark & \xmark \\
    $\Gtil_t$ & $\gtil_t(x)$ & \xmark & \cmark & \xmark \\
    $H_t$ & $h_t(x)$ & \xmark & \cmark & \cmark \\
    \hline
    \hline
  \end{tabular}
\end{table}

\subsection{Stochastic update} 
Our guarantee is on the cosine of the principal angle between the computed subspace and the ground truth eigen subspace $V$ (also called the potential function), which is a standard criterion for measuring the quality of the subspace:
\[
	\cos^2 \theta(V,G_t) = \min_w \frac{\nbr{V^\top G_t w}^2}{\nbr{G_t w}^2}.
\]

We will focus on the case when a good initialization $V_0$ is given:
\begin{align}\label{app_eqn:init}
	V_0^\top V_0 = I, ~~\cos^2 \theta(V, V_0) \geq 1/2.
\end{align}
In other words, we analyze the later stage of the convergence, which is typical in the literature (\eg, \cite{Shamir14}).
The early stage can be analyzed using established techniques (\eg, \cite{BalDasFre13}).

We will also focus on the dependence of the potential function on the step $t$. For this reason, throughout the paper we suppose $\abr{k(x,x')} \leq \kappa, \abr{\phi_\omega(x)} \leq \phi$ and regard $\kappa$ and $\phi$ as constants. Note that this is true for all the kernels and corresponding random features considered. We further regard the eigengap $\lambda_k - \lambda_{k+1}$ as a constant, which is also true for typical applications and datasets. Details can be found in the following sections.

Our final guarantee for $G_t$ is stated in the following.
\begin{oneshot}{Theorem~\ref{thm:stoch}}
Assume (\ref{app_eqn:init}) and suppose the mini-batch sizes satisfy that for any $1\leq i \leq t$,
$\nbr{A - A_i} < (\lambda_k - \lambda_{k+1})/8.$
There exist step sizes $\eta_i = O(1/i)$ such that 
\[
  1-\cos^2\theta(V, G_t) = O(1/t).
\]
\end{oneshot}

The convergence rate $O(1/t)$ is in the same order as that when computing only the top eigenvector in linear PCA~\cite{BalDasFre13}, though we are not aware of any other convergence rate for computing the top $k$ eigenfunctions in Kernel PCA.  The bound requires the mini-batch sizes are large enough so that the spectral norm of $A$ is approximated up to the order of the eigengap. This is due to the fact that approximating $A$ with $A_t$ will result in an error term in the order of $\nbr{A-A_t}$, while the increase of the potential is in the order of the eigengap. Similar terms appear in the analysis of the noisy power method~\cite{HarPri14} which, however, requires normalization and is not suitable for the kernel case. We do not specify the mini-batch sizes, but by assuming suitable data distributions, it is possible to obtain explicit bounds; see for example~\cite{Vershynin12,CaiHar12}.

\vspace{2mm}
\noindent 
{\bf Proof sketch} 
To prove the theorem, we first prove the guarantee for the normalized subspace $F_t$ which is more convenient to analyze, and then show that the update rules for $F_t$ and $G_t$ are first order equivalent so that $G_t$ enjoys the same guarantee.

\begin{oneshot}{Lemma~\ref{lem:Ft}}
  $1-\cos^2\theta(V, F_t) = O(1/t)$.
\end{oneshot}
Let $c_t^2$ denote $\cos^2\theta(V, F_t)$, 
then a key step in proving the lemma is to show that
\begin{align} \label{app_eqn:recurrence}
	c_{t+1}^2 \geq c_t^2( 1 + 2\eta_t (\lambda_k - \lambda_{k+1} - 2 \nbr{A - A_t}) (1-c_t^2) ) - O(\eta_t^2).
\end{align}
Therefore, we will need the mini-batch sizes large enough so that $2\nbr{A - A_t}$ is smaller than the eigen-gap.

Another key element in the proof of the theorem is the first order equivalence of the two update rules. To show this, we need to compare the subspaces obtained by applying the them on the same subspace $G_t$. So we introduce $F(G_t)$ to denote the subspace by applying the update rule of $F_t$ on $G_t$: 
	\begin{align*}  
      & \Ftil(G_t)  \leftarrow G_t + \eta_t A_t G_t\\
      & F(G_t) \leftarrow \Ftil(G_t)\sbr{\Ftil(G_t)^\top \Ftil(G_t)}^{-1/2}
  \end{align*}  
We show that the potentials of $G_{t+1}$ and $F(G_t)$ are close:
\begin{oneshot}{Lemma~\ref{lem:equivalence}}
$
	\cos^2\theta(V, G_{t+1}) = \cos^2\theta(V, F(G_t)) \pm O(\eta_t^2).
$
\end{oneshot}
The lemma means that applying the two update rules to the same input will result in two subspaces with similar potentials. Since $\cos^2\theta(V, F(G_t))$ enjoys the recurrence in (\ref{app_eqn:recurrence}), we know that $\cos^2\theta(V, G_{t+1}) $ also enjoys such a recurrence, which then results in $1-\cos^2\theta(V, G_t) = O(1/t)$.

The proof of the lemma is based on the observation that 
\[
  \cos^2 \theta(V, X)  =  \lambda_{\text{min}}(V^\top  X(X^\top X)^{-1} X^\top V).
\]
The lemma follows by plugging in $X=G_{t+1}$ or $X = F(G_t)$ and comparing their Taylor expansions w.r.t.\ $\eta_t$.

\subsection{Doubly stochastic update} 
For doubly stochastic update rule, the computed $H_t$ is no longer in the RKHS so the principal angle is not well defined. Since the eigenfunction $v$ is usually used for evaluating on points $x$, we will use the following point-wise convergence in our analysis. For any function $v$ in the subspace of $V$ with unit norm $\nbr{v}_\Fcal=1$, we will find a specially chosen function $h$ in the subspace of $H_t$ such that for any $x$, 
\begin{align*}
  \text{err}:=\abr{v(x) - h(x)}^2
\end{align*}
is small with high probability. More specifically, the $w$ is chosen to be $\Gtil^\top v$, and let $z = \Gtil_t w$ and $h = H_t w$.
Then the error measure can be decomposed as 
\begin{align}
  &~\abr{v(x) - h(x)}^2 \nonumber\\
  &= \abr{v(x) - z(x) + z(x) - h(x)}^2 \nonumber\\
  &\leq 2\abr{v(x) - z(x)}^2 + 2\abr{z(x) - h(x)}^2 \nonumber\\
  & \leq \underbrace{2 \kappa^2 \nbr{v - z}^2_{\Fcal}}_{\text{(I: Lemma~\ref{lem:RKHSbound})}}
  + \underbrace{2\abr{z(x) - h(x)}^2}_{\text{(II: Lemma~\ref{lem:bound})}}. \label{app_eqn:error}
\end{align}

The distance $\nbr{v - z}_{\Fcal}$ is closely related to the squared sine of the subspace angle between $V$ and $\Gtil_t$. In fact, by definition,  $\nbr{v - z}_{\Fcal}^2 = \nbr{v}_\Fcal^2 - \nbr{z}_\Fcal^2  \leq 1 - \cos^2 \theta(V, \Gtil_t)$.
Therefore, the first error term can be bounded by the guarantee on $\Gtil_t$, which can be obtained by similar arguments as for the stochastic update case. For the second term, note that $\Gtil_t$ is defined in such a way that the difference between $z(x) = \Gtil_t(x) w$ and $h(x) = H_t(x) w$ is a martingale, which can be bounded by careful analysis.

Overall, we have the following results. 
Suppose we use random Fourier features; see~\cite{RahRec08}. 
Similar bounds hold for other random features, where the batch sizes will depend on the concentration bound of the random features used.

\begin{oneshot}{Theorem~\ref{thm:doubly}}
Assume (\ref{app_eqn:init}) and suppose the mini-batch sizes satisfy that for any $1\leq i \leq t$, $\nbr{A - A_i} < (\lambda_k - \lambda_{k+1})/8$ and are of order $\Omega(\ln \frac{t}{\delta})$.
There exist step sizes $\eta_i = O(1/i)$, such that the following holds. If $\Omega(1) = \lambda_k(\Gtil_i^\top \Gtil_i) \leq \lambda_1(\Gtil_i^\top \Gtil_i) = O(1)$ for all $1\leq i\leq t$, then for any $x$ and any function $v$ in the span of $V$ with unit norm $\nbr{v}_\Fcal = 1$, we have that with probability $\geq 1- \delta$, there exists $h$ in the span of $H_t$ satisfying
\[
  |v(x) - h(x)|^2 = O\rbr{\frac{1}{t} \ln \frac{t}{\delta}}.
\]
\end{oneshot}

The point-wise error scales as $\Otil(1/t)$ with the step $t$, which is in similar order as that for the stochastic update rule.  Again, we require the spectral norm of $A$ to be estimated up to the order of the eigengap, for the same reason as before. We additionally need that the random features approximate the kernel function up to constant accuracy on all the data points up to time $t$, since the evaluation of the kernel function on these points are used in the update. This eventually leads to $\Omega(\ln \frac{t}{\delta})$ mini-batch sizes. Finally, we need $\Gtil_i^\top \Gtil_i$ to be roughly isotropic, \ie, $\Gtil_i$ is roughly orthonormal. Intuitively, this should be true for the following reasons: $\Gtil_0$ is orthonormal; the update for $\Gtil_t$ is close to that for $G_t$, which in turn is close to $F_t$ that are orthonormal. 

\vspace{2mm}
\noindent 
{\bf Proof sketch} The analysis is carried out by bounding each term in (\ref{app_eqn:error}) separately. 
As discussed above, in order to bound term I, we need a bound on the squared cosine of the subspace angle between $V$ and $\Gtil_t$. 

\begin{oneshot}{Lemma~\ref{lem:RKHSbound} }
  $1-\cos^2\theta(V, \Gtil_t) = O\rbr{\frac{1}{t} \ln \frac{t}{\delta}}$.
\end{oneshot}
To prove this lemma, we follow the argument for Theorem~\ref{thm:stoch} and get the recurrence as shown in (\ref{app_eqn:recurrence}), except with an additional error term, which is caused by the fact that the update rule for $\Gtil_{t+1}$ is using the evaluation $h_t(x_t)$ rather than $\gtil_t(x_t)$. Bounding this additional term thus relies on bounding the difference between $h_t(x) - \gtil_t(x)$, which is also what we need for bounding term II in (\ref{app_eqn:error}). For this purpose, we show the following bound:

\begin{oneshot}{Lemma~\ref{lem:bound} }
	For any  $x$ and unit vector $w$, with probability $\geq 1-\delta$ over $(\Dcal^t,\omega^t)$, 
$
  |\gtil_t(x)w - h_t(x) w|^2= O\rbr{\frac{1}{t}\ln\rbr{\frac{t}{\delta}}}.
$
\end{oneshot}
The key to prove this lemma is that our construction of $\Gtil_t$ makes sure that the difference between $\gtil_t(x)w$ and $h_t(x)w$ consists of their difference in each time step. Furthermore, the difference in each time step conditioned on previous history has mean $0$. In other words, the difference forms a martingale and thus can be bounded by Azuma's inequality.  The resulting bound depends on the mini-batch sizes, the step sizes $\eta_i$, and the evaluations $h_i(x_i)$ used in the update rules. We then judiciously choose the parameters and simplify it to the bound in the lemma. The complication of the proof is mostly due to the interweaving of the parameter values; see the following sections for the details.

\section{Stochastic Update} \label{sec:stoch}

To prove the convergence of the stochastic update rule, we first prove the convergence of the normalized version $F_t$, and then we establish the first-order equivalence of the potential functions of the two update rules for $F_t$ and $G_t$. Since the final recurrence result does not depend on higher order terms, this first-order equivalence establishes the convergence of the stochastic update rule without normalization.  

\subsection{Stochastic update with normalization}

We consider the potential function 
$ 
	1 - \cos^2\theta\rbr{V, F_{t}} 
$
and prove a recurrence for it.  We first show this for the simpler case where at each step we use the expected operator $A$ in the update rule (Lemma~\ref{lem:expected}), and then show this for the general case where $A_t$ can be different from $A$ (Lemma~\ref{lem:stoch}). Then the bound in Lemma~\ref{lem:Ft} follows from solving the recurrence in Lemma~\ref{lem:stoch}.

\subsubsection{Update rule with expected operator}
The following lemma states the recurrence for the update rule which replace $A_t$ in the stochastic update rule with the expected operator $A=\mathbb{E}A_t$:
  \begin{align}  \label{app_eqn:stoch_expected}
      &\tilde{F}_{t+1} \leftarrow F_t + \eta_t  A F_t \\ 
      &F_{t+1} \leftarrow \tilde{F}_{t+1}  \rbr{\tilde{F}_{t+1}^\top \tilde{F}_{t+1}  }^{-1/2} \nonumber
  \end{align}  

\begin{lemma} \label{lem:expected}
Let the sequence $\{F_i\}_i$ be obtained from the update rule (\ref{app_eqn:stoch_expected}), then
\[
1 - \cos^2\theta\rbr{V, F_{t+1}}  \le \sbr{1 - \cos^2\theta\rbr{V, F_{t}}} \sbr{1-2\eta_t \rbr{\lambda_k -\lambda_{k+1}} {\cos^2\theta\rbr{V, F_{t}}} } + \beta_{t},
\]
where $\beta_t = 5\eta_t^2B^2 + 3\eta_t^3 B^3$ and $\lambda_k$ and $\lambda_{k+1}$ are the top $k$ and $k+1$-th eigenvalues of $A$. 
\end{lemma}

\begin{proof}
First note that the cosine of subspace angle does not change under linear combination of the basis
\begin{align}
\cos^2 \theta\rbr{V, F_{t+1}} &= \min_{w^\prime} \frac{\nbr{V^\top F_{t+1}w^\prime}^2}{\nbr{F_{t+1}w^\prime}^2}
= \min_{w^\prime} \frac{\nbr{V^\top \tilde{F}_{t+1}  \rbr{\tilde{F}_{t+1}^\top \tilde{F}_{t+1}  }^{-1/2}w^\prime}^2}{\nbr{\tilde{F}_{t+1}  \rbr{\tilde{F}_{t+1}^\top \tilde{F}_{t+1}  }^{-1/2}w^\prime}^2}
=  \min_{w} \frac{\nbr{V^\top \tilde{F}_{t+1}w}^2}{\nbr{\tilde{F}_{t+1}w}^2}
\end{align}

The update rule gives us
\begin{align}
\nbr{V^\top \tilde{F}_{t+1}w}^2 &\ge \nbr{V^\top F_t w}^2 + 2\eta_t \inner{V^\top F_t w}{ V^\top A F_t w}
\end{align}

\begin{align}
\nbr{\tilde{F}_{t+1}w}^2 &\le \nbr{F_t w}^2 + 2\eta_t \inner{ F_t w}{A F_t w} + B\nbr{F_t w}^2\eta_t^2
\end{align}

Let $\hat{w} = w/\nbr{F_t w}$, $u = F_t \hat{w}$, so $\nbr{u} = 1$. Denote $c = \nbr{V^\top u}$ and $s  = \nbr{V^\top_{\perp}u}$. According to the definition, we have $c \ge  \cos \theta_k \rbr{V, F_t}$. Keep expanding the update rule leads to
\begin{align}
\frac{\nbr{V^\top \tilde{F}_{t+1}w}^2}{\nbr{\tilde{F}_{t+1}w}^2} &\ge
\frac{\nbr{V^\top F_t w}^2 + 2\eta_t \inner{V^\top F_t w}{V^\top AF_t w}}{\nbr{F_t w}^2 + 2\eta_t \inner{ F_t w}{A F_t w} + B\nbr{F_t w}^2\eta_t^2} \\\nonumber
&= \frac{\nbr{V^\top u}^2 + 2\eta_t \inner{V^\top u}{V^\top A u}}{1 + 2\eta_t \inner{ u}{A u} + B\eta_t^2} \\\nonumber
&\ge \cbr{ \nbr{V^\top u}^2 + 2\eta_t \inner{V^\top u}{V^\top A u} }
\cbr{1 - 2\eta_t \inner{ u}{A u} - B\eta_t^2}  \\\nonumber
&\ge \nbr{V^\top u}^2 + 2\eta_t \inner{V^\top u}{V^\top A u} - 
2\eta_t \nbr{V^\top u}^2 \inner{ u}{A u}   \\\nonumber
&  - 5\eta_t^2 B^2 -2 \eta_t^3 B^3   \\\nonumber
&= c^2 + 2\eta_t \cbr{u^\top  VV^\top Au - c^2 u^\top A u} - \beta_t  \\\nonumber
&= c^2 + 2\eta_t u^\top \rbr{VV^\top  - c^2 I} A u - \beta_t   \\\nonumber
&= c^2 + 2\eta_t u^\top \rbr{s^2 VV^\top - c^2 V_{\perp} V_{\perp} ^\top} A u - \beta_t  .
\end{align}

Recall that $A = V\Lambda_kV^\top  + V_{\perp} \Lambda_{k+1} V_{\perp}^\top$. Then
\begin{align}
u^\top \rbr{s^2 VV^\top - c^2 V_{\perp} V_{\perp} ^\top}  A u &=
s^2 u^\top V\Lambda_kV^\top u  - c^2  u^\top V_{\perp} \Lambda_{k+1} V_{\perp}^\top u \\\nonumber
&\ge  \lambda_k s^2 c^2 - \lambda_{k+1} c^2 s^2 = s^2 c^2 \rbr{\lambda_k - \lambda_{k+1}}
\end{align}

The recurrence is therefore
\begin{align}
\cos^2\theta\rbr{V, F_{t+1}} &\ge 
c^2 + 2\eta_t s^2 c^2 \rbr{\lambda_k - \lambda_{k+1}} - \beta_t   \\\nonumber
&= c^2\rbr{1 + 2\eta_t \rbr{\lambda_k - \lambda_{k+1}} \rbr{1 - c^2}}  - \beta_t.
\end{align}

The first term is a quadratic function of $c^2$:
\begin{align}
x \rbr{1 + a \rbr{1 - x}}
\end{align}
where $x:= c^2$ and $a =  2\eta_t \rbr{\lambda_k - \lambda_{k+1}}$.
It has two roots at $0$ and $1+\frac{1}{a}$. Therefore, if $\frac{1}{2}+\frac{1}{2a} \ge 1$, it is a monotonic increasing function in the interval of $\sbr{0, 1}$.

Thus, if $\eta_t \le \frac{1}{4\rbr{\lambda_k - \lambda_{k+1}}}$, which holds for all $t$ large enough, we have
\begin{align}
\cos^2\theta\rbr{V, F_{t+1}} &\ge \cos^2\theta\rbr{V, F_t}
\rbr{1 + 2\eta_t \rbr{\lambda_k - \lambda_{k+1}} \rbr{1 - \cos^2\theta\rbr{V, F_t}}}  - \beta_t
\end{align}
which leads to the lemma.
\end{proof}

\subsubsection{Using different operators in different iterations}
Now consider the case of stochastic update rule (\ref{app_eqn:stoch_orth}) where we use a mini-batch to approximate the expectation in each iteration.

\begin{lemma} \label{lem:stoch}
Let the sequence $\{F_i\}_i$ be obtained from the update rule (\ref{app_eqn:stoch_orth}), then
\[
1 - \cos^2\theta\rbr{V, F_{t+1}}  \le \sbr{1 - \cos^2\theta\rbr{V, F_{t}} } \sbr{1-2\eta_t\rbr{\lambda_k -\lambda_{k+1} - \nbr{A_t - A} } {\cos^2\theta\rbr{V, F_{t+1}} } } + \beta_{t},
\]
where $\beta_t = 5\eta_t^2B^2 + 3\eta_t^3 B^3$ and $\lambda_k$ and $\lambda_{k+1}$ are the top $k$ and $k+1$-th eigenvalues of $A$.
\end{lemma}

\begin{proof}
The effect of the stochastic update is an additional term in the recurrence
\begin{align}
\cos^2\theta\rbr{V, F_{t+1}} &\ge 
 c^2 + 2\eta_t u^\top  \rbr{s^2 VV^\top - c^2 V_{\perp} V_{\perp} ^\top} A u + Z_t - \beta_t 
\end{align}
where
\begin{align}
Z_t = 2\eta_t u^\top  \rbr{s^2 VV^\top - c^2 V_{\perp} V_{\perp} ^\top} \rbr{A_t - A} u.
\end{align}

The effect of the noise can be bounded, i.e.
\begin{align}
Z_t &= 2\eta_t s^2 u^\top  VV^\top \rbr{A_t - A} u - 2\eta_t c^2  u^\top V_{\perp}V_{\perp}^\top \rbr{A_t - A} u \\\nonumber
&= 2\eta_t s^2 u^\top  \rbr{VV^\top + l_1 I} \rbr{A_t - A} u 
- 2\eta_t c^2  u^\top \rbr{V_{\perp}V_{\perp}^\top + l_2 I} \rbr{A_t - A} u ,
\end{align}
where $s^2 l_1 = c^2 l_2$ are positive numbers such that $VV^\top + l_1 I$ and $V_{\perp}V_{\perp}^\top + l_2 I$ are positive-definite.

The generalized Rayleigh quotient leads to the inequality
\begin{align}
\left| u^\top  \rbr{VV^\top + l_1 I} \rbr{A_t - A} u \right |  &\le
\lambda u^\top  \rbr{VV^\top + l_1 I}u   \\\nonumber
&\le  \lambda \rbr{c^2 + l_1}
\end{align}
where
$\lambda$ is the largest generalized eigen-value that satisfies
\begin{align}
\rbr{VV^\top + l_1 I} \rbr{A_t - A} x = \lambda  \rbr{VV^\top + l_1 I} x.
\end{align}

Since ${VV^\top + l_1 I} $ is positive definite, we have $\lambda = \nbr{A_t - A}$.

Similarly, we have
\begin{align}
\left |  u^\top \rbr{V_{\perp}V_{\perp}^\top + l_2 I} \rbr{A_t - A} u \right|
&\le \nbr{A_t - A} \rbr{s^2 + l_2}.
\end{align}

The noise term is thus bounded by
\begin{align}
Z_t &\ge -2 \eta_t s^2  \nbr{A_t - A}\rbr{c^2 + l_1} - 2\eta_t c^2 \nbr{A_t - A}\rbr{s^2 + l_2}.
\end{align}

Note that $l_1$ and $l_2$ can be infinitely small positive so we can ignore them.

Therefore, the recurrence is
\begin{align}
\cos^2\theta\rbr{V, F_{t+1}} &\ge 
c^2 + 2\eta_t s^2 c^2 \rbr{\lambda_k - \lambda_{k+1}} -4 \eta_t \nbr{A_t - A} s^2 c^2 - \beta_t   \\\nonumber
&= c^2\rbr{1 + 2\eta_t \rbr{\lambda_k - \lambda_{k+1} - 2\nbr{A_t - A}} \rbr{1 - c^2}}  - \beta_t
\end{align}
which then leads to the lemma.
\end{proof}

In order to get fast convergence, we need to take sufficiently large mini-batches such that the variance of the noise is small enough compared with the eigen-gap.

\subsection{Stochastic update without normalization}

We show that the cosine angles of the two updates are first-order equivalent. Then, since the recurrence is not affected by higher order terms, when the step size is small enough, we can show it also converges in $O(1/t)$.

To show the first order equivalence, we need to compare the subspaces obtained by applying the them on the same subspace $G_t$. So we introduce $F(G_t)$ to denote the subspace by applying the update rule of $F_t$ on $G_t$: 
	\begin{align} \label{app_eqn:stoch_inter}  
      & \Ftil(G_t)  \leftarrow G_t + \eta_t A_t G_t\\
      & F(G_t) \leftarrow \Ftil(G_t)\sbr{\Ftil(G_t)^\top \Ftil(G_t)}^{-1/2} \nonumber
  \end{align}  
Then the first order equivalence as stated in Lemma~\ref{lem:equivalence} follows from the following two lemmas for the normalized update rule (\ref{app_eqn:stoch_orth}) and the unnormalized update rule (\ref{app_eqn:stoch_inter}), respectively.

\begin{lemma} \label{lem:equivalence_norm}
$\cos^2\theta\rbr{V, F(G_{t})} = \lambda_{\text{min}}\rbr{M + O(\eta^2) }$ where
\[
	M = V^\top PP^\top V + \eta V^\top PP^\top A V + \eta V^\top A PP^\top V - 2\eta V^\top PP^\top A PP^\top V,
\]
where $PP^\top = G_t\rbr{G_t^\top G_t}^{-1}G_t^\top$, and $P$ is an orthonormal basis for the subspace $G_t$. 
\end{lemma}

\begin{proof}
For simplicity, let $G$ denote $G_t$, and let $A$ denote $A_t$ in the following. We first have 
\begin{align}
\cos^2\theta\rbr{V, F(G)} & = \lambda_{\text{min}}\rbr{V^\top F(G) F(G)^\top V}  \\ 
&= \lambda_{\text{min}}\rbr{F(G)^\top  VV^\top F(G)} \label{app_eqn:exchange}\\
&= \lambda_{\text{min}}\cbr{V^\top \rbr{G + \eta_t A G} \sbr{\rbr{G + \eta_t A G}^\top \rbr{G + \eta_t A G} }^{-1} \rbr{G + \eta_t A G}^\top V}. \label{app_eqn:expand}
\end{align}

Note that (\ref{app_eqn:exchange}) is due to the fact that
\begin{align}
\lambda_{\text{min}}\rbr{F(G)^\top V V^\top F(G)} &=
\min_{w} \frac{w^\top F(G)^\top V V^\top F(G) w}{w^\top w}   \nonumber\\\nonumber
&= \min_{w} 
	\frac{w^\top R^{-1}\rbr{G + \eta_t A G}^\top V V^\top \rbr{G + \eta_t A G}R^{-1} w}{w^\top w} \\\nonumber
&= \min_{z} 
   \frac{z^\top\rbr{G + \eta_t A G}^\top V V^\top \rbr{G + \eta_t A G}z}{z^\top R^2 z} \\\nonumber
&= \min_{z} 
   \frac{z^\top\rbr{G + \eta_t A G}^\top V V^\top \rbr{G + \eta_t A G}z}
   {z^\top \rbr{G + \eta_t A G}^\top \rbr{G + \eta_t A G} z} \\\nonumber
&= \min_{z} 
   \frac{\nbr{V^\top \rbr{G + \eta_t A G}z}^2}
   {\nbr{\rbr{G + \eta_t A G} z}^2}
\end{align}
where $R = \sbr{\rbr{G + \eta_t A G}^\top \rbr{G + \eta_t A G}}^{1/2}$.

Now turn back to (\ref{app_eqn:expand}). Expand the matrix-valued function
\begin{align}
\phi(\eta) &=  \sbr{\rbr{G + \eta A G}^\top \rbr{G + \eta A G} }^{-1}  \\\nonumber
&= \phi(0) + \phi^\prime(0) \eta + O(\eta^2).
\end{align}

\begin{align}
\phi^\prime(0) = - 2 \rbr{G^\top G}^{-1} G^\top A G  \rbr{G^\top G}^{-1} .
\end{align}

So,
\begin{align}
\phi(\eta) &= \rbr{G^\top G}^{-1} - 2\eta \rbr{G^\top G}^{-1} G^\top A G  \rbr{G^\top G}^{-1} + O(\eta^2).
\end{align}

Therefore,
\begin{align}
& V^\top \rbr{G + \eta_t A G} \sbr{\rbr{G + \eta_t A G}^\top \rbr{G + \eta_t A G} }^{-1} \rbr{G + \eta_t A G}^\top V \\\nonumber
&=  \rbr{V^\top G + \eta_t V^\top A G} \sbr{\rbr{G^\top G}^{-1} - 2\eta \rbr{G^\top G}^{-1} G^\top A G  \rbr{G^\top G}^{-1} + O(\eta^2)} 
\rbr{G^\top V + \eta_t G^\top AV} \\\nonumber
&= V^\top G\rbr{G^\top G}^{-1}G^\top V + \eta V^\top G\rbr{G^\top G}^{-1}G^\top A V 
+ \eta V^\top A G \rbr{G^\top G}^{-1}G^\top V \\\nonumber
&- 2\eta V^\top G   \rbr{G^\top G}^{-1} G^\top A G  \rbr{G^\top G}^{-1}G^\top V + O(\eta^2) \\\nonumber
&= V^\top PP^\top V + \eta V^\top PP^\top A V 
+ \eta V^\top A PP^\top V - 2\eta V^\top PP^\top A PP^\top V + O(\eta^2),
\end{align}
where $PP^\top = G\rbr{G^\top G}^{-1}G^\top$, and $P$ is an orthonormal basis for the subspace $G$. 
\end{proof}

\begin{lemma}\label{lem:equivalence_un}
$\cos^2\theta\rbr{V, G_{t+1}} = \lambda_{\text{min}}\rbr{M}$ where $M$ is as defined in Lemma~\ref{lem:equivalence_norm}.
\end{lemma}

\begin{proof}
For simplicity, let $G$ denote $G_t$ and let $A$ denote $A_t$. Then $\cos^2\theta\rbr{V, G_{t+1}} = \lambda_{\text{min}}\rbr{N}$, where 
\begin{align*}
N = V^\top G_{t+1} \sbr{G_{t+1}^\top G_{t+1}}^{-1} G_{t+1}^\top V \text{~~with~~} G_{t+1} = G + \eta \rbr{I - GG^\top}A G.
\end{align*}

Now it suffices to show $N=M$. Consider
\begin{align*}
\phi(\eta) = \sbr{\rbr{G + \eta \rbr{I - GG^\top}A G}^\top \rbr{G + \eta \rbr{I - GG^\top} A G} }^{-1}.
\end{align*}
Then 
\begin{align*}
\phi^\prime(0) = - \rbr{G^\top G}^{-1} \sbr{G^\top\rbr{I - GG^\top}A G + G^\top A\rbr{I - GG^\top} G } \rbr{G^\top G}^{-1}
\end{align*}

Therefore, $N$ is 
\begin{eqnarray*}
& & V^\top \rbr{G + \eta \rbr{I - GG^\top}A G}
 \sbr{\rbr{G + \eta \rbr{I - GG^\top}A G}^\top \rbr{G + \eta \rbr{I - GG^\top} A G} }^{-1} \\
&& \times \rbr{G + \eta \rbr{I - GG^\top}A G}^\top V \\\nonumber
&=&  \rbr{V^\top G + \eta V^\top\rbr{I - GG^\top}A G}
 \sbr{\rbr{G + \eta \rbr{I - GG^\top}A G}^\top \rbr{G + \eta \rbr{I - GG^\top} A G} }^{-1} \\
&& \times \rbr{G^\top V + \eta G^\top A\rbr{I - GG^\top}  V} \\\nonumber
&= & \rbr{V^\top G + \eta V^\top\rbr{I - GG^\top}A G} \\
 && \times \sbr{\rbr{G^\top G}^{-1}  - \eta \rbr{G^\top G}^{-1} \sbr{G^\top\rbr{I - GG^\top}A G + G^\top A\rbr{I - GG^\top} G } \rbr{G^\top G}^{-1} } \\
 && \times \rbr{G^\top V + \eta G^\top A\rbr{I - GG^\top}  V} \\\nonumber
&=& V^\top G\rbr{G^\top G}^{-1}G^\top V   + \eta V^\top G\rbr{G^\top G}^{-1}G^\top A\rbr{I - GG^\top}V 
+ \eta V^\top \rbr{I - GG^\top}A G\rbr{G^\top G}^{-1}G^\top V   \\\nonumber
& & - \eta V^\top G\rbr{G^\top G}^{-1}  \sbr{G^\top\rbr{I - GG^\top}A G + G^\top A\rbr{I - GG^\top} G }  \rbr{G^\top G}^{-1} G^\top V   \\\nonumber
&= &V^\top PP^\top V  + \eta V^\top PP^\top A\rbr{I - GG^\top}V
+ \eta V^\top \rbr{I - GG^\top}A PP^\top V  \\\nonumber
& &- \eta V^\top PP^\top \rbr{I - GG^\top}A PP^\top  V - \eta V^\top PP^\top A\rbr{I - GG^\top} PP^\top  V   \\\nonumber
&=& V^\top PP^\top V + \eta V^\top PP^\top A V 
+ \eta V^\top A PP^\top V - 2\eta V^\top PP^\top A PP^\top V \\\nonumber
& & -\eta V^\top PP^\top A GG^\top V - \eta V^\top GG^\top A PP^\top V 
+ \eta V^\top PP^\top GG^\top A PP^\top V + \eta V^\top PP^\top A GG^\top PP^\top V  \\\nonumber
&= &V^\top PP^\top V + \eta V^\top PP^\top A V 
+ \eta V^\top A PP^\top V - 2\eta V^\top PP^\top A PP^\top V
\end{eqnarray*}
which completes the proof.
\end{proof}

\section{Doubly Stochastic Update}\label{sec:doubly}

In this section, we consider the doubly stochastic update rule. Suppose in step $t$, we use a mini-batch consisting of $B_{x,t}$ random data points $x^r_t (1\leq r \leq B_{x,t})$ and $B_{\omega,t}$ random features $\omega^s_t (1\leq s \leq B_{\omega,t})$.   Then the update rule is 
\begin{align}
H_{t+1} & = H_t + \eta_t \EE_t \sbr{\phi_{\omega_t}(x_t) \phi_{\omega_t}(\cdot) h_t(x_t)} - \eta_t H_t \EE_t\sbr{h_t(x_t)^\top h_t(x_t)} \\
 & = H_t (I - \eta_t \EE_t\sbr{h_t(x_t)^\top h_t(x_t)} )  + \eta_t\EE_t\sbr{\phi_{\omega_t}(x_t) \phi_{\omega_t}(\cdot) h_t(x_t)}
\end{align}
where for any function $f(x, \omega)$, $\EE_t f(x_t, \omega)$ denotes $\sum_{r=1}^{B_{x,t}} \sum_{s=1}^{B_{\omega,t}} f(x^r_t, \omega^s_t) / (B_{x,t} B_{\omega,t})$.
As before, we assume $H_0 = F_0$ is a good initialization, \ie, $F_0^\top F_0 = I$ and $\cos^2\theta(F_0, V) \ge 1/2$. 
Note that $H_t = [h^1_t(\cdot), \ldots, h^k_t(\cdot)]$, while $h_t(x_t)$ is its evaluation at $x_t$, \ie, $h_t(x_t)$ is a row vector $[h^1_t(x_t), \ldots, h^k_t(x_t)]$.

We introduce the following intermediate function for analysis:
\begin{align}
\Gtil_{t+1} & = \Gtil_t + \eta_t \EE_t\sbr{k(x_t,\cdot) h_t(x_t)} - \eta_t \Gtil_t \EE_t\sbr{h_t(x_t)^\top h_t(x_t)} \\
 & = \Gtil_t (I - \eta_t \EE_t\sbr{h_t(x_t)^\top h_t(x_t) } )  + \eta_t \EE_t\sbr{k(x_t, \cdot) h_t(x_t)}.
\end{align}
Again, $\Gtil_0= F_0$.

The analysis follows our intuition: we first bound the difference between $H_t$ and $\Gtil_t$ by a martingale argument, and then bound the difference between $\Gtil_t$ and $V$. For the second step we can apply the previous argument. Note that $\Gtil_t$ is different from $F_t$ since $A_t F_t = k(x_t, \cdot) F_t(x_t)$ is now replaced by $k(x_t, \cdot) h_t(x_t)$, so we need to adjust our previous analysis. 

Suppose we use random Fourier features for points in $\RR^d$; see~\cite{RahRec08}. Then we have 
\begin{oneshot}{Theorem~\ref{thm:doubly}}
Suppose the mini-batch sizes satisfy that for any $1\leq i \leq t$, $\nbr{A - A_i} < (\lambda_k - \lambda_{k+1})/8$ and $B_{x,i} = \Omega(\ln \frac{t}{\delta})$.
There exist 
 step sizes $\eta_i = O(1/i)$, such that the following holds. If $\Omega(1) = \lambda_k(\Gtil_i^\top \Gtil_i) \leq \lambda_1(\Gtil_i^\top \Gtil_i) = O(1)$ for all $1\leq i\leq t$, then for any $x$ and any function $v$ in the span of $V$ with unit norm $\nbr{v}_\Fcal = 1$, we have that with probability $\geq 1- \delta$, there exists $h$ in the span of $H_t$ satisfying
\[
  |v(x) - h(x)|^2 = O\rbr{\frac{1}{t} \ln \frac{t}{\delta}}.
\]
\end{oneshot}

\begin{proof}
Let $w  = \Gtil_t^\top v$, $z = \Gtil_t w$, and $h = H_t w$. 
\begin{align*}
  \abr{v(x) - h(x)}^2
  &= \abr{v(x) - z(x) + z(x) - h(x)}^2 \\
  &\leq 2\abr{v(x) - z(x)}^2 + 2\abr{z(x) - h(x)}^2 \\
  &\leq 2\nbr{v - z}_{\Fcal}^2 \nbr{k(x,\cdot)}_{\Fcal}^2 +  2\abr{z(x)  -  h(x)}^2 \\  
	& \leq 2\kappa^2 \nbr{v - z}_{\Fcal}^2 + \abr{z(x) - h(x)}^2.
\end{align*}
Roughly speaking, the difference between $v$ and  $z$ is the error due to random data points and can be bounded by Lemma~\ref{lem:RKHSbound1}, while the difference between $z(x)$ and $h(x)$ is the error due to random features and can be bounded by Lemma~\ref{lem:bound1}.
More precisely, since $z $ is the projection of $v$ on the span of $\Gtil_t$, 
\begin{align*}
  \nbr{v - z}^2_{\Fcal} = &  \nbr{v}_{\Fcal}^2  - \nbr{z}_{\Fcal}^2   \leq  1  - \cos^2 \theta(\Gtil_t, V) = O\rbr{ \frac{1}{t}\ln\frac{t}{\delta}}
\end{align*}
where the last step is by Lemma~\ref{lem:RKHSbound1}.
Also, since $\nbr{w} \leq 1$, we have $\abr{z(x) - h(x)}^2 = O\rbr{ \frac{1}{t}\ln\frac{t}{\delta}}$ by Lemma~\ref{lem:bound1}. 

What is left is to check the mini-batch sizes; see the assumptions in Lemma~\ref{lem:martingale} and Lemma~\ref{lem:RKHSbound1}. We need $\lambda_k( \EE_i\sbr{ h_i(x_i)^\top h_i(x_i)}) = \lambda_k( \EE_x\sbr{ h_i(x)^\top h_i(x)}) \pm O(1)$, so we only need to estimate $\EE_x\sbr{h^j_i(x)^\top h^\ell_i(x)}$ up to constant accuracy for all $1\leq j,\ell \leq k$, for which $B_{x,i} = O(\ln \frac{t}{\delta})$ suffices.  We also need $\Delta_\omega = O(\lambda_k - \lambda_{k+1}) = O(1)$, so we only need $\Delta_\omega = O(1)$. This is a bound for $(t B_{x,i})^2$ pairs of points, for which $B_{\omega,i} = O(\ln \frac{t}{\delta})$ suffices.
\end{proof}
Similar bounds hold for other random features, where the batch sizes will depend on the concentration bound of the random features used. 

The rest of this section is the proof of the theorem. For simplicity, $\nbr{\cdot}_\Fcal$ is shorten as $\nbr{\cdot}$.

First, we bound the difference between $H_t$ and $\Gtil_t$. 

%
%
\begin{lemma}\label{lem:martingale}
Suppose $|k(x,x')| \leq \kappa, |\phi_\omega(x)| \leq \phi$.  Suppose the mini-batch sizes are large enough so that $\abr{k(x_i, x_j) - \sum_{s=1}^{B_{\omega,i}} \phi_{\omega_s}(x_i) \phi_{\omega_s}(x_j)/ B_{\omega,i} } \leq \Delta_\omega$ for all sampled data points $x_i$ and $x_j$. 
For any $w$ and $x$, with probability $\geq 1-\delta$ over $(\Dcal^t,\omega^t)$, 
\[
  |\gtil_{t+1}(x) w - h_{t+1}(x) w|^2 \leq B^2_{t+1} := \frac{1}{2} \Delta_\omega^2 \ln\rbr{\frac{2}{\delta} } \sum_{i=1}^t \big| \EE_i\abr{h_i(x_i)  }  a_{t,i} w \big|^2
\]
where $a_{t,i} = \eta_i  \prod_{j=i+1}^t \rbr{I - \eta_j \EE_j\sbr{h_j(x_j)^\top h_j(x_j)}}$ for $1\leq i\leq t$, and $\abr{h_i(x_i) } := \sbr{\abr{h^j_i(x_i)}}_{j=1}^k$.
\end{lemma}

\begin{proof}
Note that
\begin{align}
H_{t+1} & = \sum_{i=1}^t \EE_i\sbr{\phi_{\omega_i}(x_i) \phi_{\omega_i}(\cdot) h_i(x_i)} a_{t,i}  + F_0 a_{t,0} , \\
\Gtil_{t+1} & = \sum_{i=1}^t \EE_i\sbr{k(x_i, \cdot) h_i(x_i) } a_{t,i} + F_0 a_{t,0},
\end{align}
where $a_{t,0} = \prod_{j=1}^t \rbr{I - \eta_j \EE_j\sbr{h_j(x_j)^\top h_j(x_j)}}$. 

We have $\gtil_{t+1}(x) w - h_{t+1}(x)w = \sum_{i=1}^t V_{t,i}(x)$ where 
\[
  V_{t,i}(x)  = \EE_i \sbr{k(x_i, x) h_i(x_i) - \phi_{\omega_i}(x_i) \phi_{\omega_i}(x) h_i(x_i)} a_{t,i}w.
\]
$V_{t,i}(x)$ is a function of $(\Dcal^i, \omega^i)$ and 
\[
 \EE_{\Dcal^i, \omega^i}\sbr{V_{t,i}(x) | \omega^{i-1}} =  \EE_{\Dcal^i, \omega^{i-1}}\EE_{\omega_i}\sbr{V_{t,i}(x) | \omega^{i-1}} = 0,
\]
so $\cbr{V_{t,i}(x)}$ is a martingale difference sequence. 

Since $|V_{t,i}(x)| < \Delta_\omega|\EE_i\abr{h_i(x_i)  } a_{t,i}w|$, the lemma follows from Azuma's Inequality.
\end{proof}

So to bound $|\gtil_t(x)w - h_t(x) w|$, we need to bound $|\EE_i\abr{h_i(x_i)  } a_{t,i}w|$, which requires some additional assumptions.

\begin{lemma}[Complete version of Lemma~\ref{lem:bound}] \label{lem:bound1}
Suppose the conditions in Lemma~\ref{lem:martingale} are true.
Further suppose for all $i\leq t$, $\eta_i = \theta/i$ where $\theta$ is sufficiently large so that $\theta \lambda_k( \EE_i\sbr{ h_i(x_i)^\top h_i(x_i)}) \geq 1$; also suppose $ \lambda_1\rbr{\Gtil_i^\top \Gtil_i} = O(1)$.
\begin{enumerate}
\item[(1)] With probability $\geq 1-\delta$ over $(\Dcal^t,\omega^t)$, for all $1\leq i \leq t$ and $\ell \in [k]$, we have
\[
  |\gtil^\ell_i(x_i) - h^\ell_i(x_i)|^2= O\rbr{\frac{\Delta_\omega^2\theta^4}{t}\ln\rbr{\frac{t}{\delta}}}.
\]
\item[(2)] For any  $x$ and unit vector $w$, with probability $\geq 1-\delta$ over $(\Dcal^t,\omega^t)$, 
\[ 
  |\gtil_t(x)w - h_t(x) w|^2= O\rbr{\frac{\Delta_\omega^2\theta^4}{t}\ln\rbr{\frac{t}{\delta}}}.
\]
\end{enumerate}
\end{lemma}

\begin{proof}
We first do induction on statement (1), which is true initially. Assume it is true for $t$, we prove it for $t+1$.

We have that for any unit vector $w$, 
\begin{align*}
|\EE_i\abr{h_i(x_i) } a_{t,i}w| & = \abr{  \eta_i  \EE_i\abr{h_i(x_i) } \prod_{j=i+1}^t \sbr{I - \eta_j \EE_j\sbr{h_j(x_j)^\top h_j(x_j)}} w} \\
& \leq \eta_i  \nbr{\EE_i\abr{h_i(x_i)}} \nbr{w} \prod_{j=i+1}^t \nbr{I - \eta_j \EE_j\sbr{ h_j(x_j)^\top h_j(x_j)}}\\
& \leq O(1) \frac{\theta^2}{i} \prod_{j=i+1}^t \rbr{1-\frac{1}{j}}  = O\rbr{\theta^2 \over t}.
\end{align*} 
We use in the second line 
\[ 
\nbr{h_i(x_i)} \leq O\rbr{\sqrt{\frac{\theta^2}{t} \ln \frac{t}{\delta}} } + \nbr{\gtil_i(x_i)} \leq O\rbr{\sqrt{\frac{\theta^2}{t} \ln \frac{t}{\delta}} } + \sqrt{ \nbr{\Gtil_i^\top \Gtil_i}} \nbr{\phi(x_i)} = O(\theta)
\] 
that holds with probability $1-t\delta/(t+1)$ by induction, 
and we use in the last line $\theta \lambda_k( \EE_i\sbr{h_i(x_i)^\top h_i(x_i)}) \geq 1$.

Then by Lemma~\ref{lem:martingale}, with probability $\geq 1-\delta/(k(t+1))$,  
\begin{align*}
|\gtil_{t+1}(x_{t+1}) w - h_{t+1}(x_{t+1}) w|^2 & \leq \frac{1}{2} \Delta_\omega^2 \ln\rbr{\frac{2(t+1)}{\delta} } \sum_{i=1}^t \big | \EE_i\abr{h_i(x_i)  }  a_{t,i} w  \big|^2\\
& \leq O(\Delta_\omega^2)  \ln\rbr{\frac{t+1}{\delta}}  \sum_{i=1}^t \frac{\theta^4}{t^2} = O\rbr{\frac{\Delta_\omega^2\theta^4}{t+1}\ln\rbr{\frac{t+1}{\delta}}}.
\end{align*} 
Repeating the argument for $k$ basis vectors $w = e_i (1\leq i\leq k)$  completes the proof.

The other statement follows from similar arguments.
\end{proof}

Next, we bound the difference between $\Gtil_t$ and $V$.

\begin{lemma}\label{lem:RKHSbound0}
Suppose the conditions in Lemma~\ref{lem:bound1} are true and furthermore, $\lambda_k(\Gtil_i^\top \Gtil_i) = \Omega(1)$ for all $i \in [t]$. Let $c^2_t$ denote $\cos^2\theta(\Gtil_t, V)$. Then with probability $\geq 1-\delta$,
\[
  c^2_{t+1} \geq c^2_t  \cbr{1 + 2\eta_t \sbr{\lambda_k - \lambda_{k+1} - 2\nbr{A_t - A} - O\rbr{\Delta_\omega\theta^2\sqrt{\frac{1}{t} \ln\frac{t}{\delta}} } } \rbr{1 - c^2_t} -  O\rbr{\eta_t\Delta_\omega\theta^2\sqrt{\frac{1-c^2_t}{t} \ln\frac{t}{\delta}} }}  - \beta_t
\]
where $\beta_t$ is as defined in Lemma~\ref{lem:expected}.
\end{lemma}
\begin{proof}
The potential of $\Gtil_t$ can be computed by a similar argument as in the previous section; the only difference is replacing  $A_t u$ with $k(x_t, \cdot) h_t(x_t)\what$. This leads to
\begin{align}
\cos^2 \theta(\Gtil_{t+1}, V)
&\ge c^2 + 2\eta_t u^\top \rbr{s^2 VV^\top - c^2 V_{\perp} V_{\perp}^\top} k(x_t, \cdot) h_t(x_t)\what - \beta_t \nonumber\\
&= c^2 + 2\eta_t u^\top \rbr{s^2 VV^\top - c^2 V_{\perp} V_{\perp}^\top} \sbr{(k(x_t, \cdot) h_t(x_t)\what - A_t u)  + (A_t u - Au) + Au} - \beta_t \label{eqn:addterm}
\end{align}
where $u = \Gtil_t \what$ with unit norm $\nbr{u} = 1$.

The terms involving $(A_t u - Au)$ and $Au$ can be dealt with as before,
so we only need to bound the extra term
\begin{align*}
  & u^\top \rbr{s^2 VV^\top - c^2 V_{\perp} V_{\perp} ^\top} [k(x_t, \cdot) h_t(x_t) \what - A_t u] \\
	= & ~ u^\top \rbr{s^2 VV^\top - c^2 V_{\perp} V_{\perp} ^\top} [k(x_t, \cdot) h_t(x_t) \what- k(x_t, \cdot) \gtil_t(x_t)\what]\\
	= & ~ u^\top \rbr{s^2 VV^\top - c^2 V_{\perp} V_{\perp} ^\top} k(x_t, \cdot) [h_t(x_t) -  \gtil_t(x_t)] \what. 
\end{align*}

So we need to bound $[h_t(x_t) -  \gtil_t(x_t)] \what$, which in turn relies on Lemma~\ref{lem:bound1}. 
More precisely, we have $\nbr{h_t(x_t) -  \gtil_t(x_t)}_\infty \leq \Otil\rbr{\Delta_\omega\theta^2 \sqrt{1/t}}$ with probability $\geq 1-\delta$. Also, we have  $u = \Gtil_t\what$ has unit norm, so $\nbr{\what} = O(1)$ when  $\lambda_k(\Gtil_i^\top \Gtil_i) = \Omega(1)$.
Then
\begin{align*}
\abr{u^\top VV^\top k(x_t, \cdot) [h_t(x_t) -  \gtil_t(x_t)] \what} 
\leq  &  \nbr{u^\top V} \nbr{k(x_t, \cdot)}  \Otil\rbr{\Delta_\omega\theta^2\sqrt{1/t}} 
\leq  c^2 \Otil\rbr{\Delta_\omega\theta^2\sqrt{1/t}}
\end{align*}
where the last step follows from $c \geq 1/2$ by assumption. 
Similarly,
\begin{align*}
\abr{u^\top V_{\perp}V_{\perp}^\top k(x_t, \cdot) [h_t(x_t) -  \gtil_t(x_t)] \what} 
\leq  &  \nbr{u^\top V_{\perp}} \nbr{k(x_t, \cdot)}  \Otil\rbr{\Delta_\omega\theta^2\sqrt{1/t}} 
\leq  s \Otil\rbr{\Delta_\omega\theta^2\sqrt{1/t}} =  \Otil\rbr{\Delta_\omega\theta^2\sqrt{\frac{1-c^2}{t}}}.
\end{align*}
Plugging into  (\ref{eqn:addterm}) and apply a similar argument as in Lemma~\ref{lem:expected} and~\ref{lem:stoch} we have the lemma. 
\end{proof}

\begin{lemma}[Complete version of Lemma~\ref{lem:RKHSbound}] \label{lem:RKHSbound1}
If the mini-batch sizes are large enough so that $\nbr{A - A_i} < (\lambda_k - \lambda_{k+1})/8$, $\lambda_k( \EE_i\sbr{ h_i(x_i)^\top h_i(x_i)}) = \lambda_k( \EE_x\sbr{ h_i(x)^\top h_i(x)}) \pm O(1)$, and $\Delta_\omega  = O(\lambda_k - \lambda_{k+1})$, then
\begin{enumerate}
\item[(1)] $\theta=O(1)$;
\item[(2)] $1- c^2_t = O\rbr{\frac{1}{t}\ln\frac{t}{\delta}}.$
\end{enumerate}
\end{lemma}

\begin{proof}
If the mini-batch size is large enough so that $\lambda_k( \EE_i\sbr{ h_i(x_i)^\top h_i(x_i)}) = \lambda_k( \EE_x\sbr{ h_i(x)^\top h_i(x)}) \pm O(1)$, we only need to show $\lambda_k( \EE_x\sbr{ h_i(x)^\top h_i(x)}) = \Omega(1)$, which will lead to $\theta=O(1)$ and then solving the recurrence in Lemma~\ref{lem:RKHSbound0} leads to $1- \cos^2 \theta(\Gtil_{t+1}, V) = \Otil(1/t)$. 

Let $e_i(x) = h_i(x) - \gtil_i(x)$. Then
\[
  \EE_x\sbr{ h_i(x)^\top h_i(x)}  =  \EE_x\sbr{ \gtil_i(x)^\top \gtil_i(x)} + 2  \EE_x\sbr{ e_i(x)^\top h_i(x)} - \EE_x\sbr{ e_i(x)^\top e_i(x)}.
\]
By Lemma~\ref{lem:bound1}, $\EE_x\abr{e^j_i(x)} = \Otil(\theta^4/t)$, which is $o(1)$ if $\theta=O(1)$. \bruce{more explanation: switching on statement (2) in the lemma}
Then the norm of $2  \EE_x\sbr{ e_i(x)^\top h_i(x)} - \EE_x\sbr{ e_i(x)^\top e_i(x)}$ is $o(1)$, so we only need to consider $\EE_x\sbr{ \gtil_i(x)^\top \gtil_i(x)}$. 

Formally, we prove our statements (1)(2) by induction. 
They are true initially.
Suppose they are true for $t-1$, we prove them for $t$.  

First, by solving the recurrence for $c_t$, we have that statement (2) is true up to step $t$.

Next, since $\EE_x \sbr{\gtil_t(x)^\top \gtil_t(x)} = \Gtil_t^\top A \Gtil_t$, we have 
\begin{align*}
w^\top\EE_x \sbr{\gtil_t(x)^\top \gtil_t(x)} w = & w^\top \Gtil_t^\top A \Gtil_t w \\
 = & w^\top \Gtil_t^\top (V\Lambda_k V^\top + V_\perp\Lambda_\perp V_\perp^\top) \Gtil_t w \\
\geq & w^\top \Gtil_t^\top V\Lambda_k V^\top  \Gtil_t w \\
\geq  & \lambda_k c_t^2 \nbr{w}^2
\end{align*}
which means $\lambda_k(\EE_x \sbr{\gtil_t(x)^\top \gtil_t(x)} ) = \Omega(1)$ by induction on $c_t$ and by the assumption that $\lambda_k(\Gtil_t^\top \Gtil_t) = \Omega(1)$.
This then leads to  $\lambda_k( \EE_i\sbr{ h_i(x_i)^\top h_i(x_i)})) = \Omega(1)$, which means $\theta = O(1)$ up to step $t$. 
\end{proof}

\end{appendix}

\end{document}